%% file: main.tex
\documentclass{article}

\usepackage{microtype}
\usepackage{graphicx}
\usepackage{subcaption}
\usepackage{booktabs} 
\usepackage{colortbl}
\usepackage{pgf}
\usepackage{hyperref}

\usepackage[preprint]{icml2026}

\usepackage{amsmath}
\usepackage{amssymb}
\usepackage{mathtools}
\usepackage{amsthm}
\usepackage{multirow}

\usepackage[capitalize,noabbrev]{cleveref}

\theoremstyle{plain}
\newtheorem{theorem}{Theorem}[section]

\theoremstyle{definition}

\theoremstyle{remark}

\usepackage[textsize=tiny]{todonotes}

\definecolor{color1}{rgb}{0.1,0.7,0.8} 
\definecolor{color2}{rgb}{0.9,0.1,0.1} 
\definecolor{color3}{rgb}{0.7,0.3,0.7} 
\definecolor{color4}{rgb}{0.3,0.3,0.7} 
\definecolor{color5}{RGB}{8, 102, 3} 
\definecolor{color6}{rgb}{0.53, 0.66, 0.42} 

\usepackage{booktabs}
\usepackage{tabularx}
\usepackage{array}
\usepackage[table]{xcolor}
\usepackage{subcaption}
\usepackage{caption}
\usepackage{enumitem}
\usepackage{amsmath}

\definecolor{sectiongray}{gray}{0.92}
\usepackage{xspace}
\newcommand{\OURS}{\textbf{\texttt{ROCKET}}\xspace}

\icmltitlerunning{\OURS: Residual-Oriented Multi-Layer Alignment for Spatially-Aware Vision-Language-Action Models}

\begin{document}

\twocolumn[
  \icmltitle{ROCKET: Residual-Oriented Multi-Layer Alignment for Spatially-Aware Vision-Language-Action Models}




\begin{icmlauthorlist}
  \icmlauthor{Guoheng Sun\textsuperscript{*}}{umd}
  \icmlauthor{Tingting Du}{wisc}
  \icmlauthor{Kaixi Feng}{umd}
  \icmlauthor{Chenxiang Luo}{cityu} \\
  \icmlauthor{Xingguo Ding}{sps}
  \icmlauthor{Zheyu Shen}{umd}
  \icmlauthor{Ziyao Wang}{umd}
  \icmlauthor{Yexiao He}{umd}
  \icmlauthor{Ang Li\textsuperscript{\ensuremath{\dagger}}}{umd}
\end{icmlauthorlist}

\icmlaffiliation{umd}{University of Maryland, College Park}
\icmlaffiliation{wisc}{University of Wisconsin, Madison}
\icmlaffiliation{cityu}{City University of Hong Kong}
\icmlaffiliation{sps}{St.\ Paul's School}

\begin{center}
$^{*}$~ghsun@umd.edu \qquad $^{\dagger}$~angliece@umd.edu
\end{center}



  \icmlcorrespondingauthor{Ang Li}{angliece@umd.edu}

\icmlkeywords{Spatial Reasoning; Vision-Language-Action Model; Representation Alignment}

  \vskip 0.3in
]



\printAffiliationsAndNotice{}  

\input{sections/0_abstract}
\input{sections/1_introduction}

\input{sections/2_related_work}
\input{sections/3_motivations}

\input{sections/4_method}

\input{sections/5_experiment}
\input{sections/6_discussion}

\input{sections/7_conclusion}

\clearpage

\section*{Impact Statement}
This paper presents work whose goal is to advance the field of Machine
Learning and Embodied AI. There are many potential societal consequences of our work, none
which we feel must be specifically highlighted here.

\bibliography{references}
\bibliographystyle{icml2026}

\newpage
\appendix
\onecolumn

\input{sections/8_appendix}

\end{document}

%% file: sections/0_abstract.tex
\begin{abstract}
Vision-Language-Action (VLA) models enable instruction-following robotic manipulation, but they are typically pretrained on 2D data and lack 3D spatial understanding. An effective approach is representation alignment, where a strong vision foundation model is used to guide a 2D VLA model. However, existing methods usually apply supervision at only a single layer, failing to fully exploit the rich information distributed across depth; meanwhile, naïve multi-layer alignment can cause gradient interference.
We introduce \textbf{\OURS}, a residual-oriented multi-layer representation alignment framework that formulates multi-layer alignment as aligning one residual stream to another. Concretely, \textbf{\OURS} employs a shared projector to align multiple layers of the VLA backbone with multiple layers of a powerful 3D vision foundation model via a layer-invariant mapping, which reduces gradient conflicts. We provide both theoretical justification and empirical analyses showing that a shared projector is sufficient and outperforms prior designs, and further propose a Matryoshka-style sparse activation scheme for the shared projector to balance multiple alignment losses. Our experiments show that, combined with a training-free layer selection strategy, \textbf{\OURS} requires \textit{only about 4\% of the compute budget while achieving 98.5\% state-of-the-art success rate on LIBERO}. We further demonstrate the superior performance of \textbf{\OURS} across LIBERO-Plus and RoboTwin, as well as multiple VLA models. The code and model weights can be found at \url{https://github.com/CASE-Lab-UMD/ROCKET-VLA}.
  
\end{abstract}

%% file: sections/1_introduction.tex
\section{Introduction}
Robotic manipulation requires models to integrate semantic intent with precise spatial reasoning about objects, contacts, and constraints in the physical world. Recent vision-language-action (VLA) models achieve strong instruction-following by leveraging large vision-language backbones, yet these backbones are predominantly pretrained on 2D imagery and often fail to produce representations with stable 3D structure~\cite{pmlr-v229-zitkovich23a,kim2024openvla,li2025spatial}. As a result, VLAs may succeed on familiar scenarios but generalize poorly when manipulation depends on accurate geometry, viewpoint changes, or fine-grained spatial relations~\cite{pmlr-v164-shridhar22a,pmlr-v229-huang23b}.

To bridge this gap, prior work has pursued three main strategies to equip VLAs with spatial reasoning: (1) augmenting inputs with explicit 3D signals such as depth maps or point clouds~\cite{pmlr-v205-shridhar23a}; (2) recovering 3D structure from 2D observations via depth estimation~\cite{ranftl2021dpt}; and (3) implicitly aligning intermediate VLA representations to spatial features produced by pretrained 3D foundation models~\cite{li2025spatial}. Compared with the first two approaches, which rely on additional sensors or external depth estimators, implicit alignment is typically more inference-efficient and easier to scale, making it a particularly practical direction for improving spatial grounding in VLAs.

In \textit{representation alignment} for improving spatial reasoning, a common practice is to supervise a VLA model’s visual-token representations to match those of a strong 3D vision foundation model. Most existing methods apply alignment at a single layer~\cite{yu2024representation, huang2025mllms,li2025spatial,wang2024reconstructive}. However, prior work has found that performance is highly sensitive to the chosen layer, and the optimal layer often varies across tasks and data distributions (as shown in Table~\ref{tab:alignment_comparison}). Since it is difficult to predict which layer contains the most appropriate visual representations \emph{a priori}, single-layer alignment typically relies on inefficient post-hoc search. Meanwhile, representations in deep residual networks evolve with depth and encode information at multiple levels~\cite{skean2025layer,lee2025vlsi}, and prior feature distillation studies similarly suggest benefits from leveraging teacher signals across depths~\cite{sun2019patient,chen2021distilling,gong-etal-2025-beyond}. A natural alternative is multi-layer alignment, which can exploit hierarchical spatial cues from shallow to deep layers. However, in the VLA setting, directly extending classic distillation by aligning multiple depths with \emph{independent} projectors per layer often backfires: separate projectors tend to learn inconsistent mappings between representation spaces, leading to severe gradient interference and ultimately causing performance collapse.

To address gradient interference, we propose \textbf{\OURS}, a stable and high-performance multi-layer alignment strategy that formulates multi-layer alignment as aligning one residual stream to another. From a residual-dynamics perspective, we attribute the failure of naive multi-projector designs to gradient conflicts, and show that using a shared projector promotes gradient coherence. Concretely, \OURS learns a cross-layer shared mapping via a lightweight projector that aligns residual streams from multiple layers to a target residual stream. This shared alignment helps reduce gradient interference, enabling fast convergence while continuing to improve late-stage performance.
We further observe that, in multi-layer alignment, shallower layers converge more easily; to prevent shallow-layer alignment from dominating the shared projector, \OURS introduces a Matryoshka-style~\cite{kusupati2022matryoshka} sparse activation scheme in which deeper layers activate more projector parameters. Such a scheme balances alignment losses across depths, encouraging shallow layers to quickly capture common local cues while allowing deeper layers to further refine global information. We also analyze layer selection and show that a simple training-free rule yields stable gains. Experiments across multiple models and datasets demonstrate that \OURS improves spatial reasoning in VLA models and consistently outperforms the single-layer alignment paradigm. Our contributions are summarized as follows:
\begin{itemize}[leftmargin=*, itemsep=0pt, topsep=2pt, parsep=2pt]
    \item We propose \OURS, a multi-layer alignment framework that injects 3D spatial reasoning into 2D-pretrained VLA models, featuring a layer-invariant shared projector to avoid gradient conflicts and a Matryoshka-style sparse activation scheme to balance alignment losses across depths.
    \item We develop theoretical and empirical analyses that explain when and why prior multi-layer alignment fails due to gradient interference, and show that a single shared projector is both sufficient and consistently superior to layer-wise projectors.
    \item By leveraging complementary spatial cues across depths, \textbf{\OURS} achieves state-of-the-art performance on LIBERO and generalizes across multiple datasets and VLA backbones.
    \item \textbf{\OURS} is compute-efficient and converges quickly, reaching state-of-the-art (SOTA) performance with $\sim$4\% of the training compute required by prior SOTA methods, while remaining effective in data-limited regimes relevant to embodied settings.
\end{itemize}

%% file: sections/2_related_work.tex
\section{Related Work}\label{sec:related_work}

\paragraph{Vision-Language-Action Models and Spatial Grounding}
VLA models integrate visual perception and language understanding to generate embodied robot actions. Representative works include \textbf{OpenVLA}~\cite{kim2024openvla} and \textbf{OpenVLA-OFT}~\cite{kim2025fine} built upon Prismatic-7B~\cite{karamcheti2024prismatic}, as well as \textbf{PI0}~\cite{black2024pi_0} and \textbf{PI0.5}~\cite{black2025pi05} based on PaliGemma~\cite{beyer2024paligemma}, which primarily rely on 2D Vision-Language Models (VLMs). These models often lack geometric understanding of the 3D world. To address this, several works augment VLA models with 3D information, including point-cloud inputs~\cite{bhat20253d,sun2025geovla} or depth-aware encodings~\cite{qu2025spatialvla}. Meanwhile, recent 3D foundation models such as \textbf{VGGT}~\cite{wang2025vggt}, \textbf{Depth Anything}~\cite{lin2025depth}, and $\pi^3$~\cite{wang2025pi} provide strong spatial representations that can offer rich supervision signals for 2D vision backbones.

\vspace{-0.2cm}

\paragraph{Representation Alignment for Spatial Supervision}
Representation alignment (feature distillation) is widely used to transfer spatial knowledge from pretrained teachers to VLM/VLA models. Most prior work aligns student visual-token features from one layer to a fixed teacher layer. For example, \textbf{Spatial Forcing (SF)}~\cite{li2025spatial} and \textbf{GLaD}~\cite{guo2025glad} align \emph{deep} vision representations to VGGT features; \textbf{REPA}~\cite{yu2024representation} aligns \emph{shallow} diffusion-model~\cite{Peebles2022DiT} representations to DINOv2-Large~\cite{oquab2023dinov2}; and \textbf{3DRS}~\cite{huang2025mllms} supervises the \emph{final-layer} visual tokens using VGGT. Although effective, these methods typically supervise a single layer and do not fully exploit the representations across depth. Moreover, as shown in Table~\ref{tab:alignment_comparison}, the choice of which layer to align is often inconsistent across methods and depends on post-hoc empirical results.

\input{tables/related_work_comparison}

\paragraph{Multi-layer Distillation and Optimization Challenges}
Early studies on knowledge distillation have shown that supervising multiple intermediate layers can lead to more effective knowledge transfer. Methods such as \textbf{Patient Knowledge Distillation}~\cite{sun2019patient}, \textbf{ReviewKD}~\cite{chen2021distilling}, and recent feature-dynamics-based approaches~\cite{gong-etal-2025-beyond} demonstrate the benefits of multi-layer alignment in deep networks. 
To enable supervision at intermediate layers, prior methods commonly attach layer-wise \emph{projectors}~\cite{chen2022improved,miles2024understanding,miles2024vkd} that reconcile mismatched feature spaces before computing alignment losses.
However, extending multi-layer distillation to VLA models introduces additional challenges: gradient interference can arise among the separate projectors, which can degrade alignment quality. 

%% file: tables/related_work_comparison.tex
\begin{table}[htb]
    \centering
    \vspace{-0.2cm}
    \caption{Comparison of representation alignment methods for spatial supervision.
    OpenVLA -- \textbf{24 / 32}: with a 32-layer backbone, using the 24th-layer features for alignment.}
    \label{tab:alignment_comparison}
    \footnotesize
    \setlength{\tabcolsep}{11pt}
    \begin{tabularx}{\linewidth}{
        lrr
    }
        \toprule
        \textbf{Method} &
        \textbf{Student} &
        \textbf{Teacher} \\
        \midrule
        SF 
        & OpenVLA -- \textbf{24 / 32}
        & VGGT -- 24 / 24 \\
        REPA
        & DiT -- \textbf{8 / 24}
        & DINOv2 -- 24 / 24 \\
        3DRS
        & LLaVA -- \textbf{32 / 32}
        & VGGT -- 24 / 24 \\
        \bottomrule
    \end{tabularx}
\vspace{-0.5cm}
\end{table}


%% file: sections/3_motivations.tex
\section{Theoretical Framework: Residual Dynamics and Gradient Coherence}\label{sec:motivatios}

\subsection{Background}
\label{sec:problem_setup}

\paragraph{VLA as student and manipulation objective.}
For ease of understanding, we view representation alignment from the perspective of feature knowledge distillation. 
We consider a pretrained VLA policy as the \emph{student} model.
Given an observation (e.g., multi-view images) $I$ and an instruction $P$, the student induces an action policy $\pi_{\theta}$.
One possible instantiation is an autoregressive policy~ \cite{kim2024openvla} that predicts an action sequence $A=\{a_t\}_{t=1}^{K}$:
\begin{equation}
a_t \sim \pi_{\theta}\!\left(\,\cdot \mid I, P, a_{<t}\right),
\label{eq:vla_policy}
\end{equation}
while our method is agnostic to the specific action parameterization (e.g., continuous regression or discretized action tokens).
The student is trained with an action learning loss
\begin{equation}
\mathcal{L}_{\mathrm{action}} = \mathcal{L}\!\left(A, A^{gt}\right),
\label{eq:action_loss}
\end{equation}
where $A^{gt}$ denotes the ground-truth action sequence from demonstrations.

\paragraph{3D vision foundation model as teacher.}
We assume access to a \emph{frozen} 3D-aware vision foundation model as the \emph{teacher}, which provides geometry-grounded \emph{patch token} representations.
Such models (e.g., VGGT~\cite{wang2025vggt}, $\pi^3$~\cite{wang2025pi}, Depth Anything 3~\cite{lin2025depth}) learn spatially consistent features from monocular or multi-view images and can serve as strong supervision signals for injecting 3D cues into a VLA student's visual representations.

\paragraph{Single-layer representation alignment.}
Single-layer alignment applies supervision on the \emph{visual tokens} of a student model by matching them to a teacher's intermediate representations at a selected depth to improve spatial awareness.
There has been a line of prior work adopting single-layer intermediate feature matching, which is simple yet effective~\cite{li2025spatial,guo2025glad}.
Let $h^{S}_{s}(\theta)\in\mathbb{R}^{d_S}$ and $h^{T}_{\tau}\in\mathbb{R}^{d_T}$ denote the student/teacher vision representations at student layer $s$ and teacher layer $\tau$.
SF aligns one student layer $s^\star$ to one teacher layer $\tau^\star$ (i.e., $\mathcal{S}=\{s^\star\}$, $\mathcal{T}=\{\tau^\star\}$) via a lightweight projector $p_{\phi}$ and optimizes
\begin{equation}
\mathcal{L}_{\mathrm{align}}
=
\ell\!\left(p_{\phi}\!\left(h^{S}_{s^\star}(\theta)\right),\, h^{T}_{\tau^\star}\right),
\label{eq:sf_align_loss}
\end{equation}
where $\ell(\cdot,\cdot)$ can be any similarity-based loss.

Accordingly, most representation-alignment methods optimize a joint objective:
\begin{equation}
\mathcal{L}_{\mathrm{total}}
\;=\;
\mathcal{L}_{\mathrm{action}}
\;+\;
\lambda\,\mathcal{L}_{\mathrm{align}},
\label{eq:sf_total_loss}
\end{equation}
where $\lambda$ balances task learning and representation alignment.

\subsection{The Gradient-Conflict of Multi-Layer Alignment}
\label{sec:multi-layer}

\paragraph{From single-layer to multi-layer alignment: richer cues, but no gains in practice.}
Despite its effectiveness, single-layer alignment has limitations: (i) \emph{single-point} alignment requires post-hoc layer selection, and (ii) it fails to use complementary cross-layer signals.
A natural extension is to align \emph{multiple} layers simultaneously, since spatial cues may be distributed across depth and multi-layer distillation is effective in other settings~\cite{skean2025layer,lee2025vlsi}.
Concretely, given ordered layer lists $\mathcal{S}=\{s_i\}_{i=1}^{N}$ (student) and $\mathcal{T}=\{\tau_i\}_{i=1}^{N}$ (teacher), previous multi-layer alignment methods use layer-specific projectors $\{p_{\phi_i}\}$:
\begin{equation}
\mathcal{L}_{\mathrm{align}}
=
\frac{1}{N}\sum_{i=1}^{N}
\ell\!\left(
p_{\phi_i}\!\left(h^{S}_{s_i}(\theta)\right),\,
h^{T}_{\tau_i}
\right).
\label{eq:naive_multilayer}
\end{equation}
However, this prior-style design~\cite{sun2019patient,chen2021distilling,gong-etal-2025-beyond} \emph{hurts} downstream VLA manipulation performance (as shown in Fig.~\ref{fig:train_stage_perf}).

\paragraph{Residual-dynamical view: multi-layer alignment should be cone-to-cone, yet learned projectors decouple.}
To understand why multi-layer alignment fails, we view both student and teacher backbones as residual dynamical systems~\cite{weinan2017proposal,chang2017multi}.
Let $h_l\in\mathbb{R}^{d}$ be the student residual representation at layer $l$:
\begin{equation}
h_{l+1} \;=\; h_l \;+\; F_l\!\left(h_l;\theta_l\right),
\qquad l=0,\dots,L-1,
\label{eq:residual_stream}
\end{equation}
with $\Delta_l:=F_l(h_l;\theta_l)$.
Empirically, the hidden states within each model tend to converge into the same feature space~\cite{gromov2024unreasonable,he2024matters} (Fig.~\ref{fig:cone_effect}), consistent with the ``cone effect''~\cite{gao2019representation}.
This suggests multi-layer alignment should learn a consistent \emph{cone-to-cone} mapping from the student residual stream to the teacher residual stream.
Yet, under naive multi-layer alignment, the learned projectors $\{p_{\phi_i}\}$ \emph{diverge} and become nearly orthogonal in parameter space (Please refer to Appendix~\ref{app:projector_similarity}), indicating that different depths follow inconsistent alignment paths rather than sharing a coherent mapping.

\begin{figure}[htbp]
    \centering
    \includegraphics[width=0.5\textwidth]{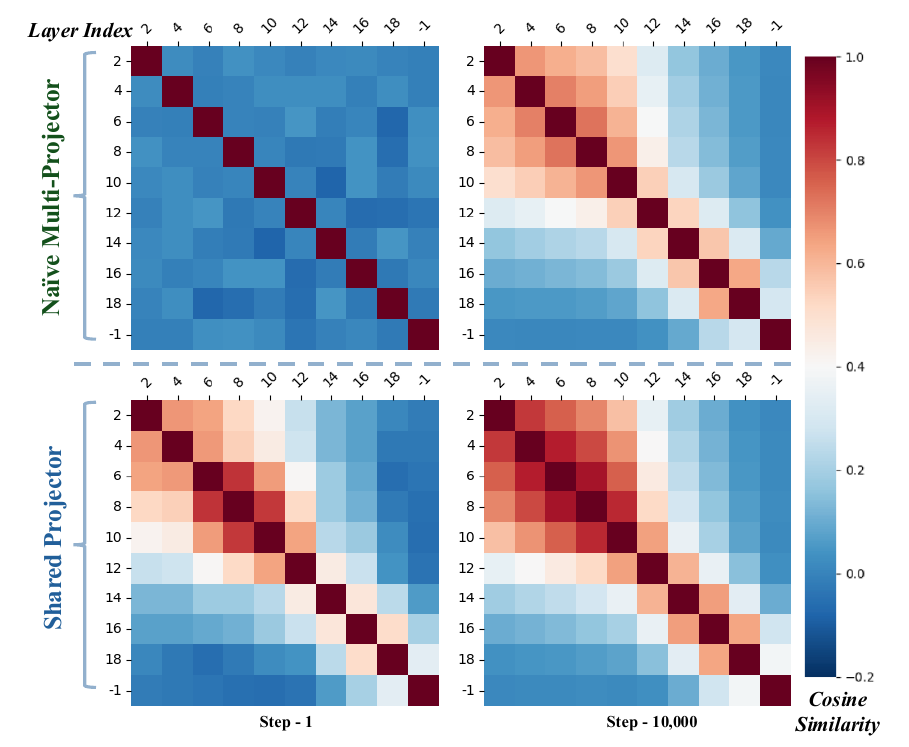} 
\caption{Cosine similarity between gradients induced by different alignment losses in the VLA's shallow layers.
\textcolor{red!70!black}{\textbf{Red}} indicates gradient coherence, while \textcolor{blue!70!black}{\textbf{Blue}} indicates gradient interference. Detailed results are provided in Fig.~\ref{fig:grad_sim_early} and Fig.~\ref{fig:grad_sim_late}.}
    \label{fig:grad}
    \vspace{-0.3cm}
\end{figure}

\paragraph{Corollary: early-layer updates are a superposition of future local distillation gradients.}
Why do inconsistent projectors matter for optimization?
Our key observation is that in a Pre-LN residual stream, back-propagation through depth is close to identity when residual updates are small, so an early-layer representation receives an \emph{additive} mixture of gradient signals coming from \emph{all} aligned future layers.
Formally, Appendix~\ref{sec:theory_shared_projector} derives (Corollary~1) that for an early layer $i$,
\begin{equation}
\begin{aligned}
\nabla_{h_i}\mathcal{L}_{\mathrm{align}}
&=\sum_{l\in\mathcal{S}:l\ge i}\alpha_l\,
\underbrace{J_{p_{\phi_l}}(h_l)^\top g_l}_{=:~v_l}
+\mathcal{E}_i, \\
g_l &:= \nabla_{z_l}\ell(z_l,t_l), \qquad
z_l=p_{\phi_l}(h_l).
\end{aligned}
\label{eq:motivation_superposition}
\end{equation}
where $\mathcal{E}_i$ is a \emph{transport error} controlled by residual-smallness and $\alpha$ denotes the layer-wise weight.
Thus, \textit{whether multi-layer alignment helps hinges on whether $\{v_l\}$ add constructively or cancel.}

\paragraph{Proposition: Jacobian-induced gradient interference under multiple independent projectors.}
Let $G \approx \sum_{l\ge i}\alpha_l v_l$ denote the superposed gradient (ignoring $\mathcal{E}_i$).
Its squared norm expands as
\begin{equation}
\|G\|_2^2
=
\sum_{l}\alpha_l^2\|v_l\|_2^2
+
\sum_{a\neq b}\alpha_a\alpha_b\,\langle v_a, v_b\rangle,
\label{eq:motivation_G_expand}
\end{equation}
where the cross terms determine constructive vs.\ destructive interference.
Crucially, each cross term is a bilinear form on teacher-side error signals:
\begin{equation}
\langle v_a, v_b \rangle
=
g_a^\top
\underbrace{\left( J_{p_{\phi_a}}(h_a)\,J_{p_{\phi_b}}(h_b)^\top \right)}_{=:~\mathcal{M}_{ab}}
g_b.
\label{eq:motivation_interference}
\end{equation}
With \emph{multiple independent} projectors, $\mathcal{M}_{ab}$ is \emph{not structurally coupled} across $(a,b)$:
the interaction matrix $J_{p_{\phi_a}}(h_a)J_{p_{\phi_b}}(h_b)^\top$ need not be symmetric or positive semidefinite (PSD), and can vary arbitrarily across layers.
As a result, $\langle v_a,v_b\rangle$ may be weak, sign-unstable, or even destructive.
Empirically, we observe that cross-layer alignment gradients remain poorly aligned (small pairwise cosine similarity) throughout training (Shown in the upper part of Fig.~\ref{fig:grad}), consistent with gradient conflict under unconstrained cross terms.

\subsection{Motivation}
\label{sec:motivation}

\paragraph{Shared projector yields a structured decomposition and improves gradient coherence in practice.}
Multi-layer alignment with multiple \emph{independent} projectors often underperforms due to gradient interference across layers.
In contrast, if we \emph{share a single projector} across layers ($\phi_l\equiv \phi$), Appendix~\ref{sec:theory_shared_projector}
shows that although $\mathcal{M}_{ab}=J_{p_\phi}(h_a)J_{p_\phi}(h_b)^\top$ is not necessarily symmetric or PSD,
it can be decomposed into a \emph{common PSD reference} plus a controlled deviation:
\begin{align}
\label{eq:motivation_shared_decomp}
\begin{aligned}
\langle v_a, v_b\rangle_{\mathrm{share}}
&= g_a^\top \mathcal{M}_{ab} g_b = g_a^\top M g_b + \Delta_{ab}, \\
&\qquad M := JJ^\top \succeq 0 .
\end{aligned}
\end{align}
Moreover, on the error-signal subspace $\mathcal{G}=\mathrm{span}\{g_l\}$ where $M$ is near-isometric, we obtain the
signal-aligned lower bound (Theorem~\ref{thm:shared_head_coherence}):
\begin{equation}
\langle v_a, v_b\rangle_{\mathrm{share}}
\;\ge\;
c\, g_a^\top g_b
\;-\;
\eta c\cdot \frac{\|g_a\|_2^2+\|g_b\|_2^2}{2}
\;-\;
|\Delta_{ab}|
\label{eq:motivation_shared_lowerbound}
\end{equation}
for constants $c>0$ and $\eta\in[0,1)$.
When $g_a^\top g_b$ is typically positive and the deviation terms are small, sharing biases cross-layer gradients
toward constructive interference. 
We validate this prediction by visualizing gradient angles: with a shared projector, cross-layer gradient cosine similarities
increase substantially compared to the independent-projector baseline (Shown in the lower part of Fig.~\ref{fig:grad}),
and this improved coherence translates into significant downstream gains on VLA manipulation tasks
(Fig.~\ref{fig:train_stage_perf}).
This motivates \textbf{\OURS}'s design choice: \emph{share the projector} to reduce gradient conflict and make multi-layer alignment effective.
Full proofs and additional empirical analysis are provided in Appendix~\ref{sec:theory_shared_projector}.

%% file: sections/4_method.tex
\section{Methodology}\label{sec:methodology}


\begin{figure*}[htbp]
    \centering
    \includegraphics[width=0.9\textwidth]{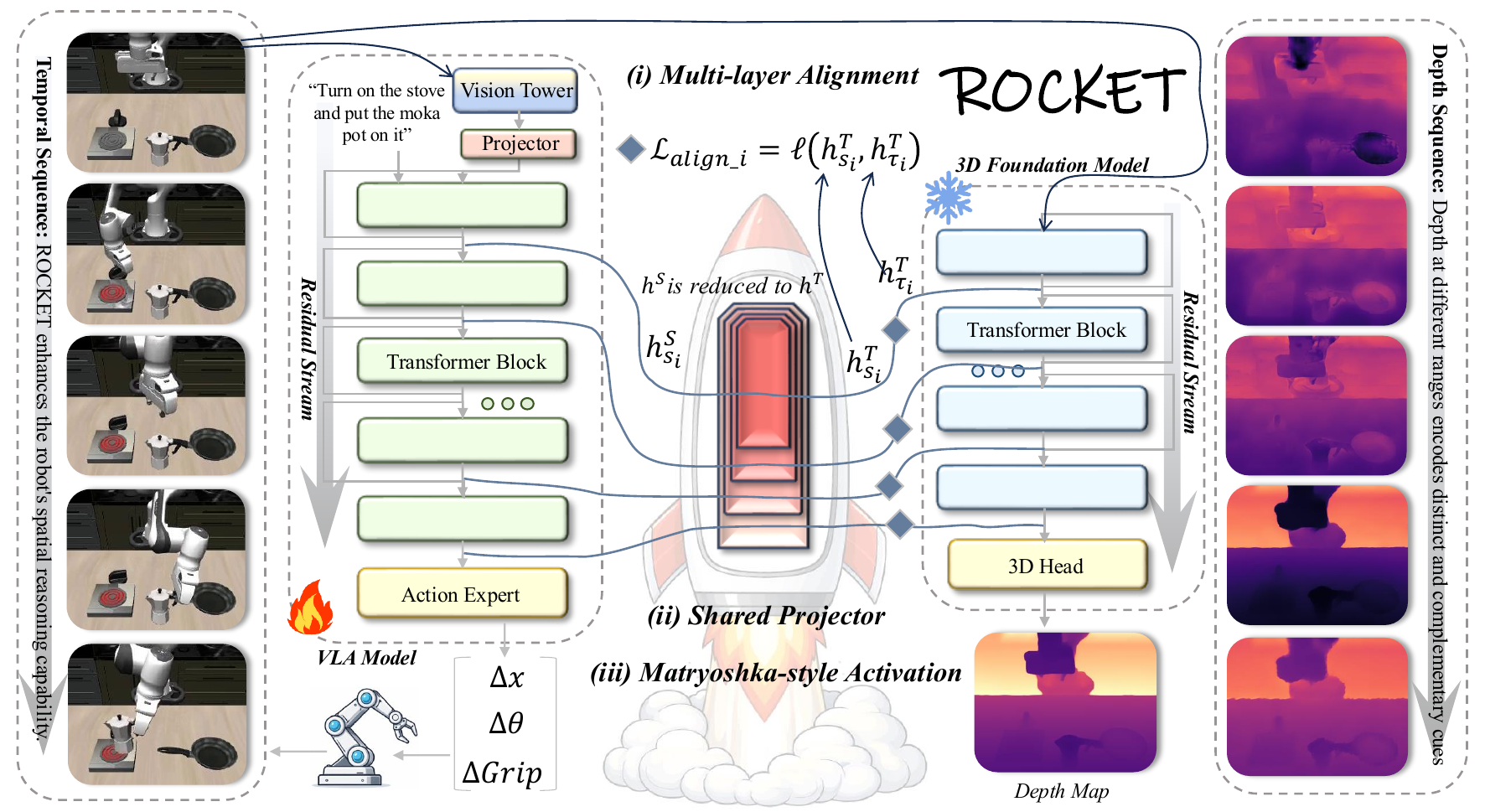} 
    \caption{ Overview of \OURS. On the \textbf{left}, we select a sample from LIBERO to showcase \OURS's performance. On the \textbf{right}, we use the outputs from VGGT layers $\{1, 3, 9, 15, 21\}$ to directly predict the depth map, demonstrating that different layers contain rich 3D information. Some of the illustrations were generated by GPT-4o~\cite{hurst2024gpt}.
    }
    \label{fig:overview}
\end{figure*}

\subsection{Overview of \OURS}
\label{sec:method-rocket}
We propose \OURS, a multi-layer representation alignment method on the residual stream.
As shown in Fig.~\ref{fig:overview}, \OURS has three key designs:

\textbf{(i) Multi-layer representation alignment.}
Single-layer alignment cannot fully leverage the rich semantic information across layers and often relies on post-hoc layer selection.
\OURS aligns multiple layers, and Sec.~\ref{sec:layer_selection} shows that with a simple layer selection strategy, \OURS outperforms the baselines.

\textbf{(ii) Shared projector to reduce gradient conflict.}
Naive multi-layer alignment may cause gradient conflicts across layers.
\OURS uses a \emph{shared} projector to learn a common mapping between two residual streams.

\textbf{(iii) Matryoshka-style activation for balancing depths.}
As indicated by Eq.~\ref{eq:motivation_superposition}, gradients from multiple alignment losses approximately add up, so shallow layers receive denser supervision.
Our experiments also show that shallow layers are easier to align (Please refer to Appendix~\ref{app:shallow_fast} for more details).
To prevent shallow layers from dominating the shared projector, \OURS introduces a Matryoshka-style~\cite{kusupati2022matryoshka} parameter-sparse activation mechanism, which progressively increases the activated parameter fraction from shallow to deep layers.

\subsection{Multi-layer Alignment with Shared Projector}

We follow the multi-layer alignment setting in Sec.~\ref{sec:multi-layer}, with $N$ aligned layer pairs
$\{(s_i,\tau_i)\}_{i=1}^{N}$ and the corresponding hidden states $\{(h^S_{s_i}, h^T_{\tau_i})\}_{i=1}^{N}$.
Crucially, we replace the layer-specific projectors $\{p_{\phi_i}\}_{i=1}^{N}$ in Eq.~\ref{eq:naive_multilayer} with a single shared projector $p_{\Phi}$,
which learns a \emph{shared} student-to-teacher mapping that is consistent across depth.
For the detailed layer selection strategy, please refer to Sec.\ref{sec:num_layer} and Sec.~\ref{sec:layer_selection}.

\begin{table*}[htbp]
\centering

\begin{minipage}[t]{0.36\textwidth}
  \vspace{0pt}
  \centering

  \includegraphics[width=\linewidth]{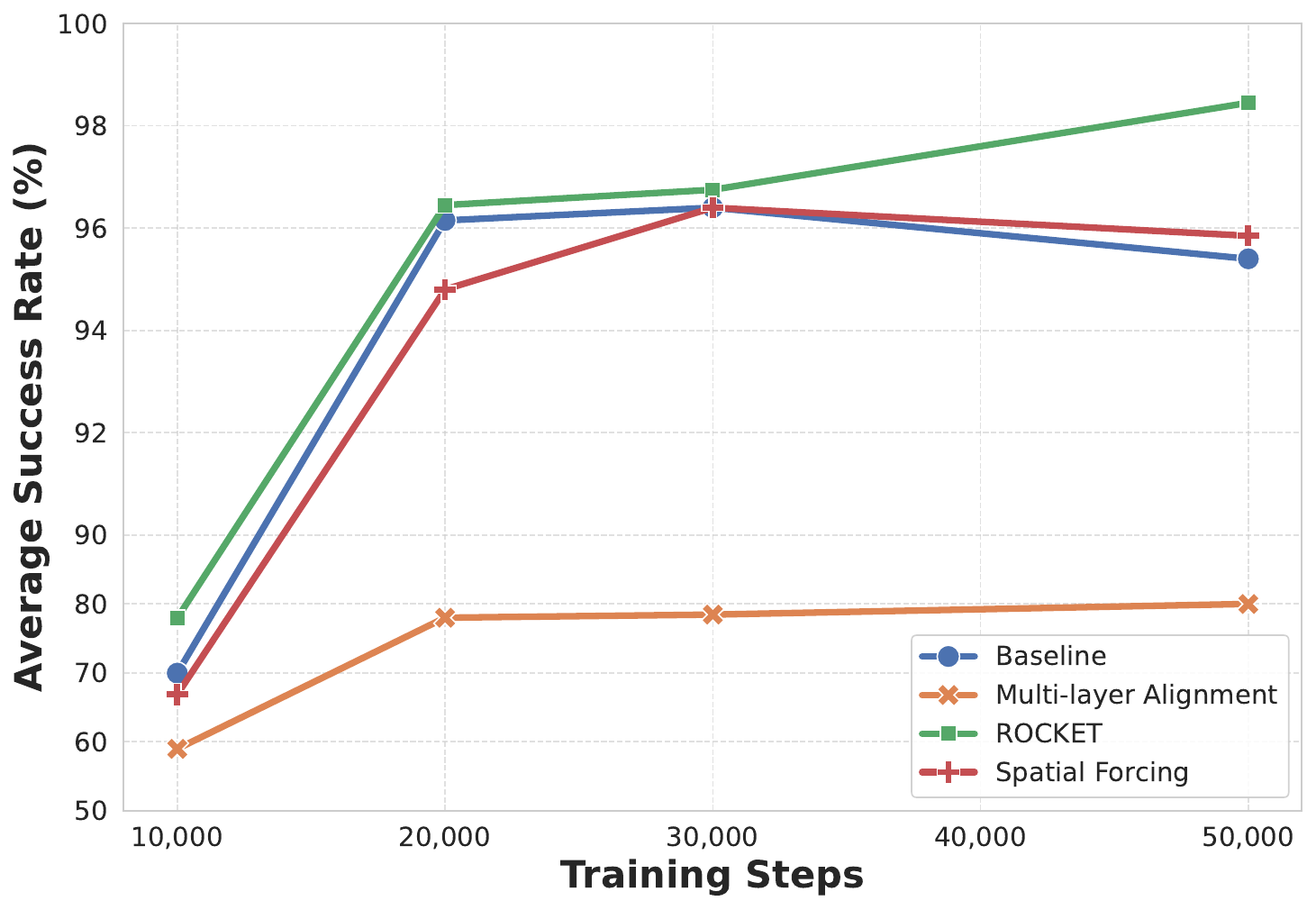}
  \captionof{figure}{Performance on LIBERO across different training stages. See Table~\ref{tab:libero_training_steps} for details.}
  \label{fig:train_stage_perf}

  \vspace{-2pt}

  \includegraphics[width=\linewidth]{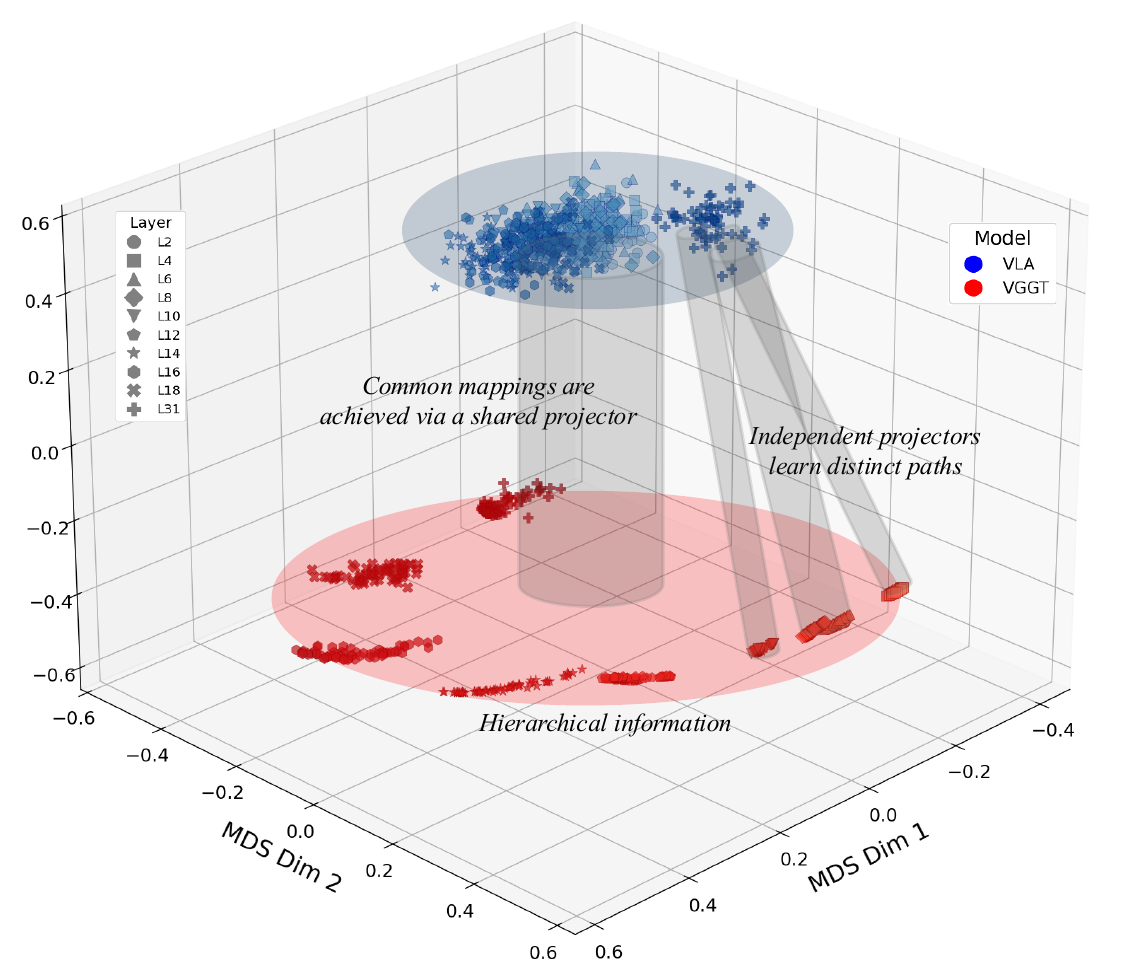}
  \captionof{figure}{A cone-effect perspective on alignment.}
  \label{fig:cone_effect}
\end{minipage}
\hfill
\begin{minipage}[t]{0.60\textwidth}
  \vspace{0pt}
  \raggedleft
  \small
  \setlength{\tabcolsep}{4pt}
    \caption{Comparisons with SOTA VLA models on the LIBERO benchmark.}
    \label{tab:sota_exps_perf}
  \resizebox{\linewidth}{!}{%
\begin{tabular}{l|cccc|c}
\toprule
\textbf{Method} &
\textbf{Spatial} &
\textbf{Object} &
\textbf{Goal} &
\textbf{Long} &
\textbf{Avg.} \\
\midrule

\multicolumn{6}{c}{\textbf{Basic 2D VLA}} \\
\midrule
\rowcolor{black!1}
Octo~\cite{ghosh2024octo}            & 78.9 & 85.7 & 84.6 & 51.1 & 75.1 \\
\rowcolor{black!2}
OpenVLA~\cite{kim2024openvla}        & 84.7 & 88.4 & 79.2 & 53.7 & 76.5 \\
\rowcolor{black!3}
WorldVLA~\cite{cen2025worldvla}      & 87.6 & 96.2 & 83.4 & 60.0 & 81.8 \\
\rowcolor{black!4}
Dita~\cite{hou2025dita}              & 84.2 & 96.3 & 85.4 & 63.8 & 82.4 \\
\rowcolor{black!7}
PI0~\cite{black2024pi_0}             & 96.8 & 98.8 & 95.8 & 85.2 & 94.2 \\
\rowcolor{black!9}
InternVLA~\cite{chen2025internvla}   & 98.0 & 99.0 & 93.8 & 92.6 & 95.9 \\
\rowcolor{black!11}
GR00T N1.6~\cite{gr00tn1_2025}       & 97.7 & 98.5 & 97.5 & 94.3 & 97.0 \\
\rowcolor{black!13}
VOTE~\cite{lin2025vote}              & 98.8 & \textbf{99.8} & 97.6 & 95.6 & 98.0 \\
\midrule
\rowcolor{black!3}
TraceVLA~\cite{zheng2024tracevla}    & 84.6 & 85.2 & 75.1 & 54.1 & 74.8 \\
\rowcolor{black!4}
ThinkAct~\cite{huang2025thinkact}    & 88.3 & 91.4 & 87.1 & 70.9 & 84.4 \\
\rowcolor{black!7}
UniVLA~\cite{bu2025univla}           & 96.5 & 96.8 & 95.6 & 92.0 & 95.2 \\
\rowcolor{black!7}
MemoryVLA~\cite{shi2025memoryvla}
& 98.4 & 98.4 & 96.4 & 93.4 & 96.5 \\
\rowcolor{black!8}
CronusVLA~\cite{shi2025memoryvla}
& 97.3 & 99.6 & 96.9 & 94.0 & 97.0 \\
\rowcolor{black!8}
OpenVLA-OFT~\cite{kim2025fine}       & 97.6 & 98.4 & 97.9 & 94.5 & 97.1 \\

\midrule
\multicolumn{6}{c}{\textbf{3D VLA with explicit 3D inputs}} \\
\midrule
\rowcolor{black!1}
SpatialVLA~\cite{qu2025spatialvla}   & 88.2 & 89.9 & 78.6 & 55.5 & 78.1 \\
\rowcolor{black!8}
GeoVLA~\cite{sun2025geovla}          & 98.4 & 99.0 & 96.6 & 96.6 & 97.7 \\
\rowcolor{black!12}
3D-CAVLA~\cite{bhat20253d}           & 98.2 & \textbf{99.8} & 98.2 & 96.1 & 98.1 \\

\midrule
\multicolumn{6}{c}{\textbf{3D VLA with representation alignment}} \\
\midrule
\rowcolor{black!8}
GLaD~\cite{guo2025glad}              & 95.0 & 97.4 & 94.4 & 89.4 & 94.1 \\
\rowcolor{black!16}
Spatial Forcing~\cite{li2025spatial} & \textbf{99.4} & 99.6 & \textbf{98.8} & 96.0 & \textbf{98.5} \\
\rowcolor{black!16}
\OURS (Ours)                          & 98.2 & \textbf{99.8} & \textbf{98.8} & \textbf{97.0} & \textbf{98.5} \\
\bottomrule
\end{tabular}
  }
\end{minipage}
\vspace{-0.5cm}
\end{table*}

\subsection{Matryoshka-style Sparse Activation}
\label{sec:nested_projector}

We parameterize a single shared projector
$p_{\Phi}:\mathbb{R}^{d_S}\rightarrow\mathbb{R}^{d_T}$ with a maximum internal width $m$.
Specifically, we present a two-layer MLP projector:
\begin{equation}
p_{\Phi}(x)
=
W_2\,\sigma(W_1 x),
W_1\in\mathbb{R}^{m\times d_S},\;
W_2\in\mathbb{R}^{d_T\times m},
\label{eq:shared_projector_full}
\end{equation}
where $\sigma(\cdot)$ is a pointwise nonlinearity (e.g., GELU), and $\Phi=\{W_1,W_2\}$.

For each aligned layer $s_i$, we activate only the \emph{first} $m_i$ hidden channels of the shared projector, where $m_i$ increases monotonically with depth.
By default, we use a simple linear width schedule:
\begin{equation}
\rho_i
=
\rho_{\min} + (\rho_{\max}-\rho_{\min})\cdot \frac{i-1}{\max(N-1,\,1)}, 
m_i \;=\; \left\lceil \rho_i\, m \right\rceil,
\label{eq:nested_width_schedule}
\end{equation}
where $\rho_{\min}=0.2$ and $\rho_{\max}=1.0$, $m$ is the maximum hidden width of the shared projector.
When $N=1$, this reduces to $\rho_1=\rho_{\max}$ (i.e., activating the full projector). Equivalently, define a binary gating vector $g_i\in\{0,1\}^{m}$ as a prefix mask
\begin{equation}
g_i[j] \;=\; \mathbb{I}\!\left(j \le m_i\right),
\qquad j=1,\dots,m,
\label{eq:prefix_gate}
\end{equation}
and apply it to the hidden activation:
\begin{equation}
p_{\Phi}^{(i)}(x)
=
W_2\!\left(g_i \odot \sigma(W_1 x)\right),
\label{eq:nested_projector}
\end{equation}
where $\odot$ denotes element-wise multiplication.
This construction yields a \emph{Matryoshka} (nested) family of projectors
$\{p_{\Phi}^{(i)}\}_{i=1}^{N}$ that share parameters, but expose increasing effective capacity for deeper layers.

\paragraph{Training objective.}
Our final training objective combines the task loss for action prediction with the proposed multi-layer alignment regularizer:
\begin{equation}
\mathcal{L}_{\textbf{\OURS}}
=
\mathcal{L}_{\mathrm{action}}
+
\lambda\,
\frac{1}{N}\sum_{i=1}^{N}
\ell\!\left(
p_{\Phi}^{(i)}\!\left(h^{S}_{s_i}\right),\,
h^{T}_{\tau_i}
\right),
\label{eq:rocket_total_loss}
\end{equation}
where $\mathcal{L}_{\mathrm{action}}$ is defined in Eq.~\eqref{eq:action_loss}, and
$\lambda$ balances task learning and representation alignment.
By default, we use $\ell(a,b)=1-\cos(a,b)$ and set $\lambda=0.5$.

%% file: sections/5_experiment.tex
\section{Main Result}\label{sec:results}

\begin{figure*}[t]
  \centering
  \includegraphics[width=1.01\linewidth]{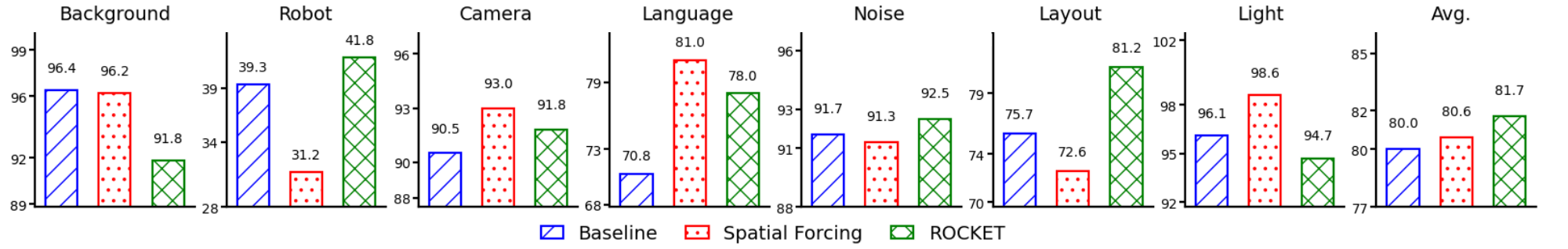}
  \caption{Success rate (\%) on LIBERO-Plus. Results are reported under the seven perturbations defined by LIBERO-Plus. Each column shows the average success rate over four task groups: Spatial, Object, Goal, and Long. Detailed results are provided in Appendix~\ref{app:libero_plus_details}.}
    \vspace{-0.3cm}
  \label{fig:liberoplus}
\end{figure*}

\subsection{Experimental Setup}
\label{sec:exp_setup}

\paragraph{LIBERO and LIBERO-Plus.}
We conduct experiments on \textsc{LIBERO}~\cite{liu2023libero} and \textsc{LIBERO-Plus}~\cite{fei2025libero}. Each benchmark comprises four sub-datasets (task suites). Following the official protocol of \textsc{LIBERO-Plus} as well as the common practice used in \textsc{PI0}/\textsc{PI0.5}, we mix the four sub-datasets during training and evaluate a \emph{single} model on the four corresponding test suites. Unless otherwise specified, we fine-tune the \textsc{OpenVLA-7B} pretrained weights on \textsc{LIBERO} and \textsc{LIBERO-Plus}. We adopt the training setting of \textsc{SF} and \textsc{OpenVLA-OFT}. Specifically, we perform LoRA fine-tuning for $50{,}000$ steps with batch size $32$.
In addition, to compare multiple methods under the \emph{full fine-tuning} regime as in \textsc{PI0.5}, we run full fine-tuning on \textsc{LIBERO} for $30{,}000$ steps with batch size $64$.

\paragraph{RoboTwin 2.0.}
We follow the official training protocol and perform LoRA fine-tuning for $30{,}000$ steps with a batch size of $16$. For evaluation, we follow the \textsc{RoboTwin 2.0}~\cite{chen2025robotwin} evaluation procedure with a random seed of $0$, and run $100$ trials for both the easy and hard task sets. We use the task subset specified in SF for both training and evaluation. For all benchmarks, we report task execution success rate (SR, \%) by default.

\begin{table}[t]
\centering
\caption{
Success rate (\%) of the PI0.5 with different representation alignment strategies
on the LIBERO.
}
\label{tab:pi05_libero_30k}
\small
\begin{tabular}{c|cccc|c}
\toprule
\textbf{Method} & \textbf{Spatial} & \textbf{Object} & \textbf{Goal} & \textbf{Long} & \textbf{Avg.} \\
\midrule
\rowcolor{black!4}
Baseline & 96.4 & 98.2 & 95.0 & 82.2 & 93.0 \\
\rowcolor{black!8}
Spatial Forcing & \textbf{97.8} & 97.8 & 94.4 & 85.8 & 94.0 \\
\rowcolor{black!16}
\textbf{\OURS} & 96.4 & \textbf{98.8} & \textbf{96.6} & \textbf{89.2} & \textbf{95.3} \\
\bottomrule
\end{tabular}

\end{table}

\paragraph{Projector and Alignment Hyperparameters.}
To ensure a fair comparison, we keep the projector architecture \emph{identical across all experiments}. In particular, we use a \emph{single} projector whose structure and dimensionalities exactly match \textsc{SF}.
For the alignment-loss weight $\lambda$ and the number of layers used by the single-layer alignment baseline in \textsc{SF}, we directly use the best configurations reported by their hyperparameter search.
More training hyperparameters and implementation details are provided in Appendix~\ref{app:training_details}.


\subsection{Performance of \OURS Across Different Settings}

\paragraph{LIBERO with OpenVLA-7B.}
\label{subsec:libero_openvla}

As shown in Table~\ref{tab:sota_exps_perf}, we compare representative approaches on LIBERO, including:
(i) standard 2D VLA models, covering both direct fine-tuning of the base VLA and variants that boost performance via additional modules;
(ii) 3D VLA methods with explicit 3D inputs, including augmenting inputs with explicit 3D signals and recovering 3D structure via depth estimation; and
(iii) two single-layer representation alignment baselines.
Results show that \textbf{\OURS} substantially outperforms prior models.
Even compared to Spatial Forcing, \textbf{\OURS} matches its performance while using only $\sim$4\% of the compute budget (please refer to Table~\ref{tab:sota_exps_cost} for more details).

\paragraph{LIBERO with PI0.5.}
\label{subsec:libero_pi05}

On PI0.5, we compare \textbf{\OURS} with the direct fine-tuning baseline and Spatial Forcing under the full fine-tuning setting.
As reported in Table~\ref{tab:pi05_libero_30k}, under the same setting, \textbf{\OURS} improves performance by $2.3\%$. Overall, this improvement suggests that even for smaller-scale models under full fine-tuning, \textbf{\OURS} can effectively inject spatial reasoning capability.

\paragraph{LIBERO-Plus with OpenVLA-7B.}
\label{subsec:liberoplus_openvla}

We further evaluate \textbf{\OURS} on LIBERO-Plus, a robustness benchmark with seven perturbation dimensions, and find that \OURS is more robust under challenging spatial perturbations. As shown in Fig.~\ref{fig:liberoplus}, \OURS achieves the best average success rate (81.7\%), outperforming the baseline (80.0\%). Notably, \OURS shows the most pronounced gains under \emph{Robot} and \emph{Layout} shifts, which are strongly tied to spatial geometry rather than appearance. These improvements suggest that \textbf{\OURS} enhances spatial reasoning instead of relying on positional shortcuts.

\paragraph{RoboTwin 2.0 with PI0.}
\label{subsec:robotwin_pi0}

For bimanual tasks, we evaluate on RoboTwin~2.0 using ALOHA assets~\cite{aldaco2024aloha}.
As shown in Fig.~\ref{fig:robotwin}, \OURS achieves a clear advantage in the Easy setting, while in the Hard setting it is only slightly behind the strongest baseline (finetuned).

\begin{figure*}[t]
  \centering
  \includegraphics[width=1.01\linewidth]{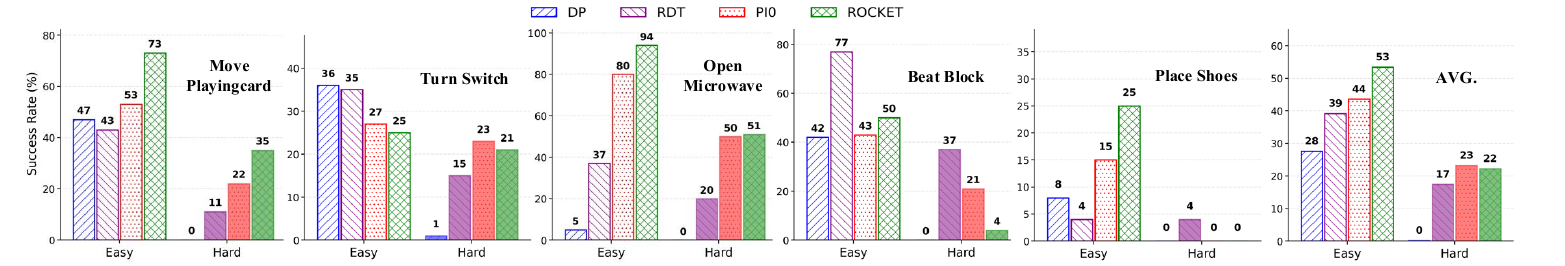}
  \vspace{-0.5cm}
  \caption{Success rate (\%) on RoboTwin~2.0 (ALOHA) under Easy and Hard settings across five bimanual tasks, comparing \OURS with DP~\cite{chi2025diffusion}, RDT~\cite{liu2024rdt}, and PI0~\cite{black2024pi_0}.}
  \label{fig:robotwin}
  \vspace{-0.4cm}
\end{figure*}

\input{tables/SOTA}

%% file: sections/6_discussion.tex
\section{Discussion}\label{sec:discussion}

\subsection{Optimal Number of Alignment Layers}
\label{sec:num_layer}
A key design choice in multi-layer alignment is the number of aligned layer pairs.
To ensure a fair comparison, we fix the projector capacity and study how many layers should be aligned under a fixed-size projector.
Following the strategy in ~\cite{sun2019patient}, we uniformly sample $N$ layers from the VLA LLM backbone and $N$ layers from VGGT, and align them as $N$ layer pairs.
As shown in Table~\ref{tab:openvla_layers}, aligning roughly $8$--$16$ layer pairs yields the best trade-off, and performance saturates or degrades beyond this range.
Therefore, \OURS uses $N=10$ aligned layer pairs by default in all experiments.

\begin{table}[htbp]
    \centering
    \caption{The impact of the number of aligned layers on final performance, evaluated on LIBERO.}
    \label{tab:openvla_layers}
    \small
    \setlength{\tabcolsep}{7.7pt}
    \begin{tabular}{c|cccc|c}
        \toprule
        \textbf{\# Layers} & \textbf{Spatial} & \textbf{Object} & \textbf{Goal} & \textbf{Long} & \textbf{Avg.} \\
        \midrule
        \rowcolor{black!8}
        4  & 98.0 & 98.8 & 93.0 & \textbf{95.2} & 96.3 \\
        \rowcolor{black!16}
        8  & \textbf{98.2} & \textbf{99.8} & 96.4 & 95.0 & \textbf{97.4} \\
        \rowcolor{black!12}
        16 & 97.8 & 99.6 & \textbf{97.6} & 91.6 & 96.7 \\
        \rowcolor{black!4}
        24 & 93.2 & 95.6 & 97.4 & 90.0 & 94.1 \\
        \bottomrule
    \end{tabular}

    \vspace{-0.5cm}
\end{table}

\subsection{Criteria for Layer Selection}
\label{sec:layer_selection}
We investigate several layer selection strategies. We apply the same strategy to select the same layers for both VGGT and VLA. Specifically, we consider:
(i) the \underline{Uniform} sampling strategy in Sec.~\ref{sec:num_layer};
(ii) selecting $10$ layers separately from the \underline{Shallow}, \underline{Middle}, and \underline{Deep} stages;
(iii) \underline{E2M-Last1}: uniformly sampling from early-to-middle layers and additionally including the last layer; and
(iv) two importance-based heuristics: Motivated by the hypothesis that layers with higher input--output cosine similarity are less informative~\cite{gromov2024unreasonable,he2024matters}, we design two similarity-based criteria.
Specifically, the first criterion computes the cosine similarity between the input and output representations aggregated over windows of $10$ consecutive Transformer blocks; \underline{Sim-1-Top} selects the $10$ layers with the highest similarity, while \underline{Sim-1-Last} selects those with the lowest similarity.
The second one computes similarity per block; \underline{Sim-2-Top} selects the $10$ highest-similarity layers, whereas \underline{Sim-2-Last} selects the $10$ lowest-similarity layers.

Results in Table~\ref{tab:ablation_study} show that even without expensive post-hoc search, most reasonable selection strategies provide consistent gains, suggesting that \OURS is robust to layer choice.
Unless otherwise stated, we adopt the \underline{E2M-Last1} strategy as the default setting.
Implementation details and the exact layer indices are provided in Appendix~\ref{app:layer_selection_strategy}, \ref{app:details_index}.

\begin{table}[htbp]
\vspace{-0.2cm}
\centering
\small
\caption{Comparison of layer selection strategies on LIBERO. }
\label{tab:ablation_study}
\setlength{\tabcolsep}{7.1pt}
\begin{tabular}{l|cccc|c}
\toprule
\textbf{Strategy} & \textbf{Spatial} & \textbf{Object} & \textbf{Goal} & \textbf{Long} & \textbf{Avg.} \\ 
\midrule
\rowcolor{black!6}
Baseline & 97.2 & 98.4 & 95.6 & 88.6 & 95.0 \\
\hline
\rowcolor{black!16}
Uniform-8 & \textbf{98.2} & 99.8 & 96.4 & \textbf{95.0} & \textbf{97.4} \\
\rowcolor{black!14}
E2M-Last1 & 96.2 & 98.0 & \textbf{97.2} & 93.6 & 96.3 \\
\hline
\rowcolor{black!6}
Shallow & 97.8 & 99.8 & 95.4 & 85.6 & 94.7 \\
\rowcolor{black!4}
Middle & 89.4 & 99.8 & 94.2 & 91.2 & 93.7 \\
\rowcolor{black!11}
Deep & 92.0 & 99.8 & 95.6 & 93.6 & 95.3 \\
\hline
\rowcolor{black!4}
Sim-1-Top & 92.8 & 99.8 & 95.0 & 86.8 & 93.6 \\
\rowcolor{black!11}
Sim-1-Last & 93.8 & \textbf{100.0} & 95.6 & 92.8 & 95.6 \\
\rowcolor{black!10}
Sim-2-Top & 91.6 & \textbf{100.0} & 96.0 & 94.0 & 95.4 \\
\rowcolor{black!9}
Sim-2-Last & 91.8 & \textbf{100.0} & 93.6 & 94.8 & 95.1 \\

\bottomrule
\end{tabular}
\vspace{-0.2cm}
\end{table}

\input{tables/discussion_layer_selection_strategy}

\subsection{Data and Computational Efficiency}
We evaluate the efficiency of \OURS from three perspectives: 
\textbf{(i) Limited-data training.}
Using $1\%$, $5\%$, and $10\%$ random subsets of LIBERO (each trained for $30$k steps), Table~\ref{tab:ratio_ablation} shows that \OURS still achieves strong performance, indicating high data efficiency, which is particularly important in robotics where data collection is costly.
\textbf{(ii) Convergence speed.}
In Fig.~\ref{fig:train_stage_perf}, \OURS outperforms the baseline by $\sim$10\% at around $10$k steps and continues to improve thereafter, suggesting that the proposed multi-level alignment mechanism more effectively exploits spatial cues across different layers. Please refer to Table~\ref{tab:libero_training_steps} for the complete results.
\textbf{(iii) Computational cost to reach SOTA.}
Table~\ref{tab:sota_exps_cost} summarizes the compute overhead of same-sized models reported in Table~\ref{tab:sota_exps_perf}. \OURS reaches SOTA performance while the previous best method incurs $24\times$ higher cost, demonstrating strong overall computational efficiency among post-training methods.

\begin{table}[htbp]
\centering
\small
\caption{Limited-data training on LIBERO, SR(\%).}
\label{tab:ratio_ablation}
\setlength{\tabcolsep}{8.3pt}
\begin{tabular}{l|cccc|c}
\toprule
\textbf{Ratio} & \textbf{Spatial} & \textbf{Object} & \textbf{Goal} & \textbf{Long} & \textbf{Avg.} \\ 
\midrule
\rowcolor{black!2}
1\%  & 75.0 & 88.8 & 65.0 & 65.6 & 73.6 \\
\rowcolor{black!8}
5\%  & 94.2 & 95.0 & 75.2 & 69.0 & 83.4 \\
\rowcolor{black!8}
10\% & 94.0 & 91.8 & 78.6 & 72.6 & 84.3 \\
\rowcolor{black!16}
100\% & \textbf{96.2} & \textbf{98.0} & \textbf{97.2} & \textbf{93.6} & \textbf{96.3} \\
\bottomrule
\end{tabular}
\end{table}

\begin{table}[htbp]
\caption{Training cost comparison of same-size models on LIBERO. 
\textbf{Cost} is estimated as $\#\text{Model} \times \text{BatchSize} \times \text{Steps}$.}
\label{tab:sota_exps_cost}
\centering
\small
\setlength{\tabcolsep}{9.8pt}
\begin{tabular}{l | c r l }
\toprule
\textbf{Method} &
\textbf{Avg.} &
\textbf{Cost} & \textbf{Details} \\
\midrule

\rowcolor{black!2}
ThinkAct

& 84.4
& $6.0\times$
& $1 \times128  \times75 $k \\
\rowcolor{black!4}
UniVLA
& 95.2
& $0.96\times$
& $1 \times 192 \times 8$k \\
\rowcolor{black!8}
OpenVLA-OFT
& 97.1
& $25.6\times$
& $4 \times 64 \times 150$k \\

\hline

\rowcolor{black!9}
GeoVLA        & 97.7 & $3.2\times$ & $1 \times 256 \times 20$k \\
\rowcolor{black!12}
3D-CAVLA         & 98.1 & $3.0\times$ & $4 \times 8 \times 150$k \\

\hline
\rowcolor{black!16}
Spatial Forcing
& \textbf{98.5}
& $24.0\times$
& $4 \times 64 \times 150$k \\
\rowcolor{black!16}
\OURS
& \textbf{98.5}
& $1.0\times$
& $1 \times 32 \times 50$k
 \\

\bottomrule
\end{tabular}
\vspace{-0.4cm}
\end{table}

\subsection{Ablation Study}

We ablate the key components of \OURS.
Introducing multi-layer alignment with \emph{independent} projectors degrades performance to $80.0\%$, indicating gradient interference across layers.
Replacing independent projectors with the proposed \emph{shared} projector improves performance to $98.2\%$.
Finally, adding the Matryoshka-style sparse activation further increases performance to $98.5\%$.

\begin{table}[htbp]
\caption{Ablation study of \OURS on LIBERO.}
\label{tab:ablation}
\centering
\small
\begin{tabular}{l|cccc|c}
\toprule
\textbf{Method} & \textbf{Spatial} & \textbf{Object} & \textbf{Goal} & \textbf{Long} & \textbf{Avg.} \\ 
\midrule
\rowcolor{black!4}
Baseline 
& 96.9 & 98.7 & 97.6 & 94.8 & 96.4 \\

\rowcolor{black!1}
+ Multi-layer
& 93.6 & 99.2 & 42.2 & 85.0 & 80.0 \\

\rowcolor{black!12}
+ Shared
& \textbf{99.0} & \textbf{99.8} & 97.0 & 96.8 & 98.2 \\
\rowcolor{black!16}
+ Matryoshka
& 98.2 & \textbf{99.8} & \textbf{98.8} & \textbf{97.0} & \textbf{98.5} \\
\bottomrule
\end{tabular}
\vspace{-0.4cm}
\end{table}

\input{tables/ablation_study_MRL}

%% file: tables/discussion_layer_selection_strategy.tex


%% file: tables/ablation_study_MRL.tex

%% file: sections/7_conclusion.tex
\section{Conclusion}\label{sec:conclusions}

We presented \textbf{\OURS}, a residual-oriented multi-layer representation alignment framework that transfers 3D spatial priors from a strong vision foundation model into 2D-pretrained VLA models. By using a \emph{layer-invariant shared projector}, \OURS mitigates gradient interference that arises in naïve multi-layer distillation, and our Matryoshka-style sparse activation further balances supervision across depths. Extensive experiments show that \OURS consistently improves spatial grounding and generalization across datasets and VLA backbones, achieving near state-of-the-art performance on LIBERO with only a small fraction of the compute. We hope \OURS offers a simple and scalable route to more reliable 3D-aware robotic manipulation.

%% file: sections/8_appendix.tex
\section{Training Details}
\label{app:training_details}

\paragraph{Teacher Freezing.}
Throughout the entire training pipeline, we keep the teacher model frozen and only fine-tune the student VLA model.

\paragraph{Learning Rate Schedule and Batch Size.}
We train \textsc{OpenVLA-7B} for a total of $50{,}000$ steps, following its default learning-rate decay strategy with a minor adjustment: we keep the learning rate at $5\times 10^{-4}$ for the first $10{,}000$ steps, and then decay it to $5\times 10^{-5}$ for the remainder of training. We use a batch size of 32.

For the search-oriented experiments in Sec.~\ref{sec:num_layer}, Sec.~\ref{sec:layer_selection} and Table~\ref{tab:ratio_ablation}, to reduce computational cost we decrease the batch size to 16. Meanwhile, to avoid under-training, we keep the learning rate at $5\times 10^{-4}$ for the first $30{,}000$ steps and then decay it to $5\times 10^{-5}$ afterwards.

\paragraph{PI-Model Experiments.}
For experiments related to PI0.5, we set the batch size to 64 and train the model in PyTorch; all other settings follow the default configuration.
For PI0 experiments on \textsc{RoboTwin 2.0}, we use a batch size of 16 and train with the JAX implementation.

We also observe that in such single-task settings, the action loss typically converges quickly. Therefore, we reduce the coefficient of the alignment loss from its default value of 0.5 to 0.125, which yields strong empirical performance.

\section{Performance Trajectories on LIBERO Across Training}
Table~\ref{tab:libero_training_steps} reports success rates on the four LIBERO suites (Spatial, Object, Goal, and Long) at different training stages, together with the overall average. Unlike single-point evaluation at the end of training, this table highlights performance trajectories over the full training course, enabling a more complete comparison of training dynamics across methods.

\begin{table*}[htbp]
\centering
\small
\caption{Success rates (\%) on the LIBERO benchmarks. We compare the performance trajectories of multiple methods over the full training course; \textbf{bold} indicates the best final result, and \underline{underlining} indicates the best result achieved during training.}
\label{tab:libero_training_steps}
\begin{tabular}{c|c|cccc|c}
\toprule
\textbf{Method} & \textbf{Steps} & \textbf{Spatial} & \textbf{Object} & \textbf{Goal} & \textbf{Long} & \textbf{Avg.} \\ 
\midrule
\multirow{4}{*}{Baseline} & 10k & 74.6 & 89.4 & 78.6 & 37.2 & 70.0 \\
                           & 20k & 93.0 & \underline{99.6} & 97.6 & \underline{94.4} & 96.2 \\
                           & 30k & \underline{94.6} & 99.2 & \underline{98.2} & 93.6 & \underline{96.4} \\
                           & 50k & 94.0 & 97.2 & 98.0 & 92.4 & 95.4 \\
\midrule
\multirow{4}{*}{Spatial Forcing} & 10k & 83.4 & 99.8 & 52.6 & 31.6 & 66.9 \\
                           & 20k & 94.2 & 99.8 & 94.8 & 90.4 & 94.8 \\
                           & 30k & \underline{97.8} & \underline{\textbf{100.0}} & 92.6 & \underline{95.2} & \underline{96.4} \\
                           & 50k & \underline{97.8} & 98.8 & \underline{96.2} & 90.6 & 95.9 \\
\midrule
\multirow{4}{*}{Multi-layer Alignment} & 10k & 86.2 & 74.6 & 22.4 & 52.6 & 59.0 \\
                            & 20k & 93.8 & 99.8 & 35.0  & 83.4 & 78.0 \\
                            & 30k & \underline{94.6} & \underline{\textbf{100.0}} & 37.4 & 81.8 & 78.5 \\
                            & 50k & 93.6 & 99.2 & \underline{42.2} & \underline{85.0} & \underline{80.0} \\

\midrule
\multirow{4}{*}{\OURS (Shared-only)} & 10k & 94.2 & 95.2 & 70.4 & 70.2 & 82.5 \\
                           & 20k & 94.0 & 99.6 & 97.4 & 94.6 & 96.4 \\
                           & 30k & 97.0 & 99.6 & \underline{98.6} & 92.4 & 96.9 \\
                           & 50k & \underline{\textbf{99.0}} & \underline{99.8} & 97.0 & \underline{96.8} & \underline{98.2} \\
\midrule
\multirow{4}{*}{\OURS} & 10k & 90.8 & 91.4 & 74.0 & 55.6 & 78.0 \\
                       & 20k & 93.4 & \underline{\textbf{100.0}} & 98.0 & 94.4 & 96.5 \\
                       & 30k & 96.0 & 99.2 & 98.0 & 93.8 & 96.8 \\
                       & 50k & \underline{98.2} & 99.8 & \underline{\textbf{98.8}} & \underline{\textbf{97.0}} & \underline{\textbf{98.5}} \\

\bottomrule
\end{tabular}
\end{table*}

\section{Dataset Details}

\subsection{LIBERO}

LIBERO~\cite{liu2023libero} is a benchmark suite for language-conditioned robotic manipulation, designed to evaluate multi-task imitation learning and generalization. The original LIBERO release consists of \textbf{12} manipulation tasks grouped into \textbf{4} task suites, instantiated across \textbf{3} robot embodiments. Tasks are defined with natural language instructions, enabling direct study of vision-language-action (VLA) policy learning.

The four task suites capture complementary generalization axes. LIBERO-SPATIAL focuses on spatial configuration variation, LIBERO-OBJECT emphasizes object-level diversity, LIBERO-GOAL introduces goal variation within shared task structures, and LIBERO-LONG contains long-horizon, multi-stage manipulation tasks requiring sequential reasoning~\cite{liu2023libero}. This structured design enables controlled evaluation of cross-task and cross-goal generalization.

LIBERO provides expert demonstration trajectories collected in simulation using RoboSuite. The original dataset contains approximately \textbf{20{,}000} demonstration episodes in total, with each task comprising roughly \textbf{1{,}500--2{,}000} trajectories depending on task complexity. Each episode consists of temporally dense state-action sequences spanning hundreds of timesteps.

Each episode includes synchronized multimodal observations, including RGB images from fixed camera viewpoints and robot proprioceptive states such as joint positions and gripper status. Actions are specified in continuous control space, typically as end-effector pose commands. Data are stored in a standardized trajectory format, facilitating efficient loading and integration with common imitation learning pipelines.

\subsection{LIBERO-Plus}

LIBERO-Plus~\cite{fei2025libero} is an extended version of the original LIBERO benchmark, designed to scale up task diversity and data volume while preserving the same language-conditioned manipulation setting. LIBERO-Plus significantly expands the number of tasks and demonstrations, enabling evaluation of large-capacity vision-language-action (VLA) models under more data-rich and heterogeneous training regimes.

Compared to the original LIBERO benchmark, LIBERO-Plus includes \textbf{over 100} manipulation tasks spanning a broader range of object configurations, goal specifications, and scene layouts. Tasks are defined using natural language instructions and are organized to support both multi-task learning and compositional generalization. The expanded task set increases structural and semantic diversity, making LIBERO-Plus more suitable for training and evaluating foundation-scale robotic policies.

LIBERO-Plus provides a substantially larger collection of expert demonstration trajectories generated in simulation. In total, the dataset contains \textbf{hundreds of thousands of demonstration episodes}, with each task comprising thousands of trajectories. Each episode consists of temporally dense state-action sequences, supporting long-horizon policy learning and sequence modeling.

Each episode includes synchronized multimodal observations, including RGB visual inputs, robot proprioceptive states, and continuous control actions specified in end-effector space. Data are stored in the same standardized trajectory format as LIBERO, enabling seamless reuse of training pipelines and direct comparison across LIBERO and LIBERO-Plus benchmarks.

\subsection{RoboTwin 2.0}

RoboTwin 2.0~\cite{chen2025robotwin}  is a large-scale simulation benchmark and dataset for bimanual robotic manipulation, providing both an automated data generator and unified evaluation protocols. The official release instantiates \textbf{50 dual-arm tasks} across \textbf{5 robot embodiments}, targeting robust long-horizon manipulation with diverse object interactions. RoboTwin 2.0 is built on an annotated object library (RoboTwin-OD) containing \textbf{731 object instances} spanning \textbf{147 categories}, enabling systematic task variation in object geometry, placement, and scene composition.

In addition to the generator, RoboTwin 2.0 provides \textbf{over 100,000} pre-collected expert trajectories/episodes for downstream training and benchmarking. Each episode consists of temporally aligned state-action sequences with multimodal observations. Visual observations include rendered RGB images from calibrated cameras (often multi-view) to capture both global context and fine-grained contact-relevant details. Proprioceptive observations include robot joint states (e.g., positions/velocities) and gripper states. Actions are specified in continuous control space (e.g., end-effector or joint-space commands), which makes the dataset compatible with standard imitation learning and policy learning pipelines.

Overall, RoboTwin 2.0 is designed to support scalable training and controlled evaluation of general-purpose policies, particularly for studying cross-task and cross-embodiment generalization under instruction- or goal-conditioned settings.

\section{Model Details}

\subsection{OpenVLA}
OpenVLA~\cite{kim2024openvla} is a 7B-parameter VLA based on the \textbf{Prismatic-7B VLM}~\cite{karamcheti2024prismatic} framework. The architecture's vision tower employs a fused visual encoder (totaling 600M parameters) constructed by combining two pretrained vision models: \textbf{DINOv2}~\cite{oquab2023dinov2} and \textbf{SigLIP}~\cite{zhai2023sigmoid}. Specifically, input image patches are passed separately through both encoders, and the resulting feature vectors are concatenated channel-wise. These fused visual features are then mapped into the language embedding space via a small 2-layer Multi-Layer Perceptron \textbf{(MLP) projector}. The central reasoning engine is a 7B-parameter \textbf{Llama 2}~\cite{touvron2023llama} language model backbone, which processes the interleaved visual tokens and language instructions to predict 7-dimensional robot actions. It features \textbf{32 Transformer layers} with \textbf{a hidden dimension of 4096} and \textbf{32 attention heads}. OpenVLA represents robot actions by discretizing them into \textbf{256 bins per dimension}, treating them as specialized tokens within the LLM's vocabulary.

\subsection{OpenVLA-OFT}
OpenVLA-OFT~\cite{kim2025fine} is a fine-tuning variant of OpenVLA that inherits the same base components. Instead of discretizing actions, OpenVLA-OFT employs a \textbf{continuous action prediction head}—a 2-block MLP ResNet that directly predicts 7-dimensional continuous actions from the LLM's hidden states via L1 regression. During fine-tuning, OpenVLA-OFT freezes the vision tower and projector and applies \textbf{LoRA}~\cite{lora2021} adapters to the LLM backbone, targeting all linear layers in self-attention and MLP blocks. OpenVLA-OFT+ extends it with \textbf{FiLM}~\cite{perez2018film} for enhanced language grounding. Specifically, FiLM applies learnable affine transformations to the visual features based on the language instruction, enabling the model to modulate visual representations according to task-specific language cues.

\subsection{PI0}

PI0\cite{black2024pi_0} is a vision-language-action (VLA) model that generates robot action sequences using flow matching. The model employs a dual-expert architecture: a vision-language expert and an action expert. The vision tower uses \textbf{SigLIP So400m/14}~\cite{zhai2023sigmoid}. The language model is \textbf{PaliGemma}~\cite{beyer2024paligemma} based on Gemma 2B~\cite{team2024gemma}, with \textbf{hidden dimension 2048, 18 transformer layers}. The action expert is a \textbf{Gemma 300M} model with hidden dimension 1024, 18 layers. During training, PI0 learns a velocity field $v_t(x)$ via \textbf{flow matching}.

\subsection{PI0.5}

PI0.5~\cite{black2025pi05} shares the same architectural foundation as PI0 but introduces two key modifications for improved open-world generalization through knowledge insulation. First, instead of providing the robot state as a continuous vector in the suffix, PI0.5 discretizes the state into \textbf{256 bins} over $[-1, 1]$ and incorporates it as discrete tokens in the language prefix. Second, PI0.5 replaces MLP-based timestep fusion with adaptive RMS normalization (\textbf{AdaRMSNorm}) for timestep conditioning. The timestep embedding is processed through a two-layer MLP with Swish activations to produce a 1024-dimensional condition vector. This condition is injected into the action expert through AdaRMSNorm layers in each of the 18 transformer layers, where AdaRMSNorm dynamically generates scale and shift parameters, which provides more flexible timestep conditioning compared to the MLP fusion approach in PI0.

\section{Projector Similarity}
\label{app:projector_similarity}
\begin{figure}[t]
    \centering
    \includegraphics[width=0.55\linewidth]{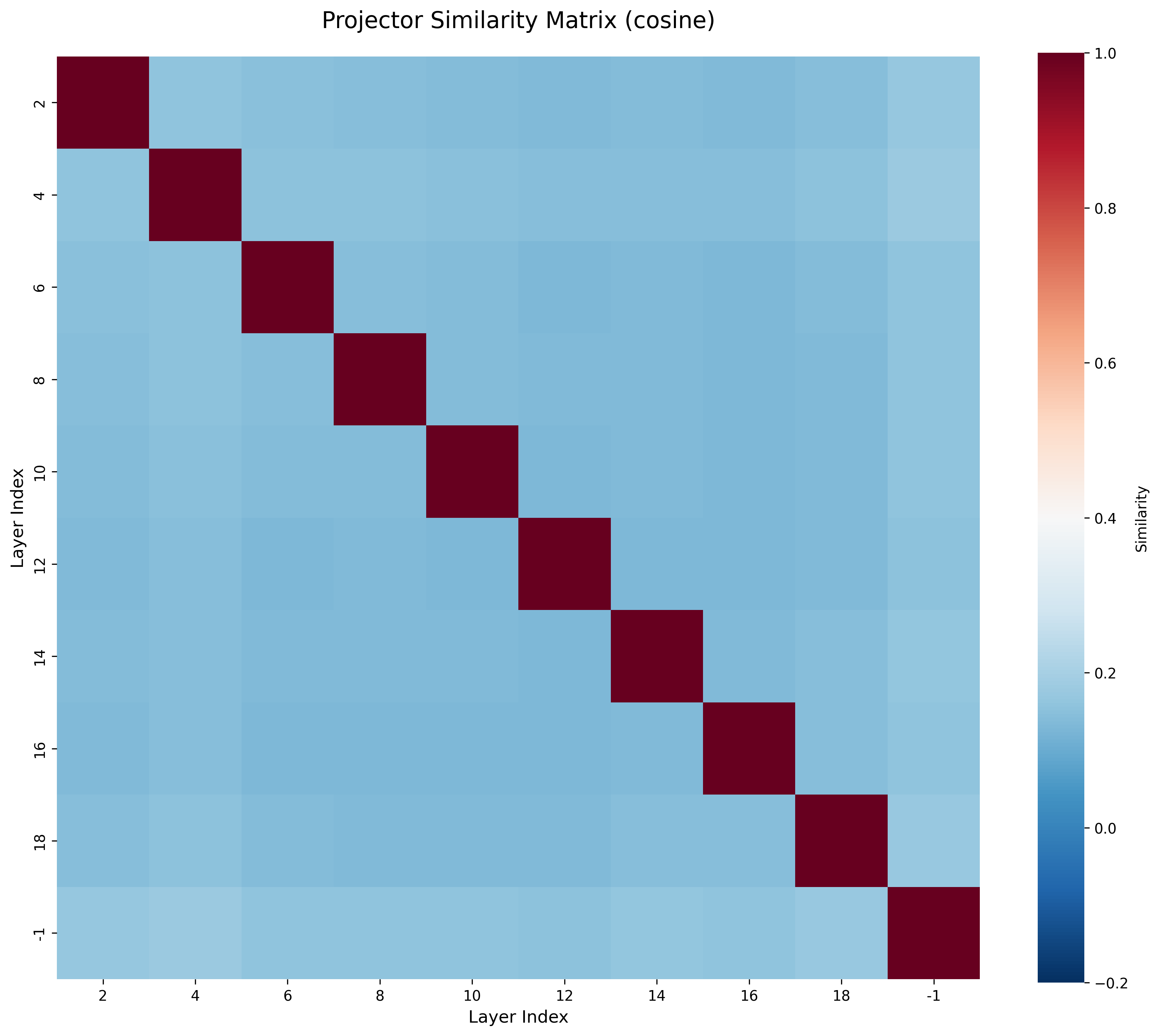}
    \caption{Pairwise cosine similarity of flattened MLP parameters for independently initialized projectors. The similarities are nearly zero, suggesting that naive multi-layer alignment with multiple independent projectors learns distinct mapping trajectories, which can induce gradient conflicts and harm final performance.}
    \label{fig:projector_cosine_similarity}
\end{figure}

We flatten the MLP parameters of each independently initialized projector (i.e., $\phi=\{W_1, W_2\}$) and compute the pairwise cosine similarities. As shown in Fig.~\ref{fig:projector_cosine_similarity}, the similarities between different projectors are nearly zero, indicating that naive multi-layer alignment with multiple independent projectors tends to learn distinct mapping trajectories when modeling the transformation from one residual stream to another. This discrepancy can lead to gradient conflicts, which in turn degrades the final performance.

\section{Why We Need a Nested Sparse-Activation Mechanism to Balance Multiple Alignment Losses}
\label{app:shallow_fast}

We compute the similarity between the student model OpenVLA-7B and the teacher model VGGT using centered kernel alignment (CKA)~\cite{cortes2012algorithms}. As shown in Fig.\ref{fig:cka_vla_vggt}, we observe that the shallow layers of VGGT are more similar to OpenVLA-7B, while the similarity gradually decreases as the depth increases. This finding not only suggests that different layers of VGGT contain semantic information with varying levels of richness, but also indicates that it is easier to align VGGT and OpenVLA-7B at shallow layers. In practice, a shared projector typically learns a mapping that works well for aligning shallow layers; however, it often fails to achieve equally good alignment for deeper layers.

\begin{figure}[htbp]
    \centering
    \includegraphics[width=0.85\linewidth]{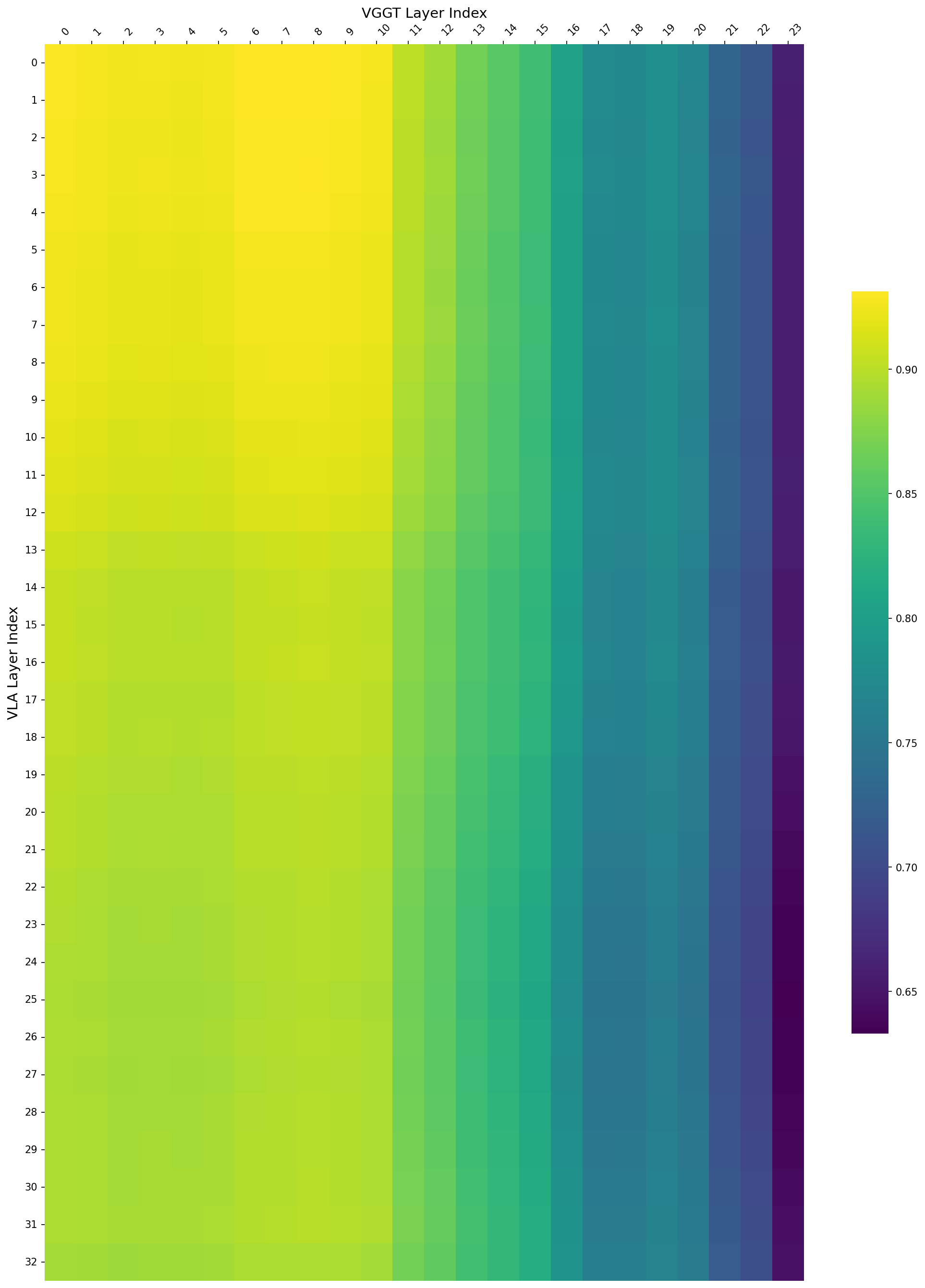}
    \caption{\textbf{Layer-wise CKA similarity between OpenVLA-7B and VGGT.} 
    We compute linear CKA between OpenVLA-7B vision-token representations and VGGT visual features across layers.
    Higher values indicate more similar representations. Shallow VGGT layers exhibit higher similarity to OpenVLA-7B, and similarity generally decreases with depth.}
    \label{fig:cka_vla_vggt}
\end{figure}

To address this multi-layer alignment mismatch, we propose a nested sparse-activation mechanism. The key idea is to let shallow layers activate (and thus update/use) fewer parameters in the projector, and gradually increase the number of activated parameters as the layer depth increases, until reaching the full projector capacity at the deepest layers.

\paragraph{CKA principle.}
CKA measures the similarity between two representation spaces by comparing their induced pairwise similarity structure, and can be viewed as a normalized version of HSIC (Hilbert-Schmidt Independence Criterion). Concretely, suppose we collect representations from one layer of the student model as a matrix $X \in \mathbb{R}^{n \times d_x}$ and from one layer of the teacher model as $Y \in \mathbb{R}^{n \times d_y}$, where $n$ is the number of matched observations (e.g., the same $n$ tokens or image regions), while $d_x$ and $d_y$ can be different.

A common choice is \emph{linear CKA}, computed via Gram matrices:
\[
K = X X^\top,\quad L = Y Y^\top,
\]
and centering them with the centering matrix $H = I - \frac{1}{n}\mathbf{1}\mathbf{1}^\top$:
\[
\tilde K = H K H,\quad \tilde L = H L H.
\]
Then CKA is
\[
\mathrm{CKA}(X,Y)
= \frac{\langle \tilde K, \tilde L \rangle_F}{\|\tilde K\|_F \, \|\tilde L\|_F},
\]
where $\langle A,B\rangle_F = \mathrm{tr}(A^\top B)$ is the Frobenius inner product.
Equivalently, for centered $X$ and $Y$, linear CKA can be written as:
\[
\mathrm{CKA}(X,Y)
= \frac{\|X^\top Y\|_F^2}{\|X^\top X\|_F \, \|Y^\top Y\|_F}.
\]
Importantly, the above formulation does \emph{not} require $d_x = d_y$; it only requires that $X$ and $Y$ share the same number of rows $n$ (i.e., aligned observations). Intuitively, CKA compares how similarly the two models organize the same set of examples in representation space, and it is invariant to isotropic scaling (and, in the linear case, robust to orthogonal transformations), making it suitable for cross-model, cross-dimension similarity analysis.

\paragraph{Details of computing CKA similarity between OpenVLA-7B and VGGT.}
Since OpenVLA-7B and VGGT have different hidden dimensions, we adopt CKA because it can compare representations of different feature dimensions as long as they are evaluated on the same set of samples (tokens/images), i.e., the same number of observations.

In our implementation, we measure representation similarity between OpenVLA-7B (student) and VGGT (teacher) using \emph{linear CKA} on a held-out set of samples. For each input batch, we extract the vision-token hidden states from selected layers of OpenVLA-7B and the corresponding visual features from selected layers of VGGT. Since the two models have different spatial resolutions and hidden dimensions, we first resample the VGGT features to match the spatial layout of OpenVLA’s vision tokens, then flatten both representations into per-sample feature vectors and concatenate all samples to form matrices $X\in\mathbb{R}^{N\times D_x}$ and $Y\in\mathbb{R}^{N\times D_y}$. We then compute linear CKA for every layer pair by constructing and centering Gram matrices $K=XX^\top$ and $L=YY^\top$, and reporting the normalized HSIC score $\mathrm{CKA}(X,Y)=\langle \tilde K,\tilde L\rangle_F/(\|\tilde K\|_F\|\tilde L\|_F)$, which is comparable even when $D_x\neq D_y$.

\section{Per-task Results on LIBERO-Plus}
\label{app:libero_plus_details}
Detailed per-task results corresponding to Fig.~\ref{fig:liberoplus} are provided in Figs.~\ref{fig:libero_plus_spatial}--\ref{fig:libero_plus_long}.

\begin{figure*}[htbp]
  \centering
  \includegraphics[width=1.01\linewidth]{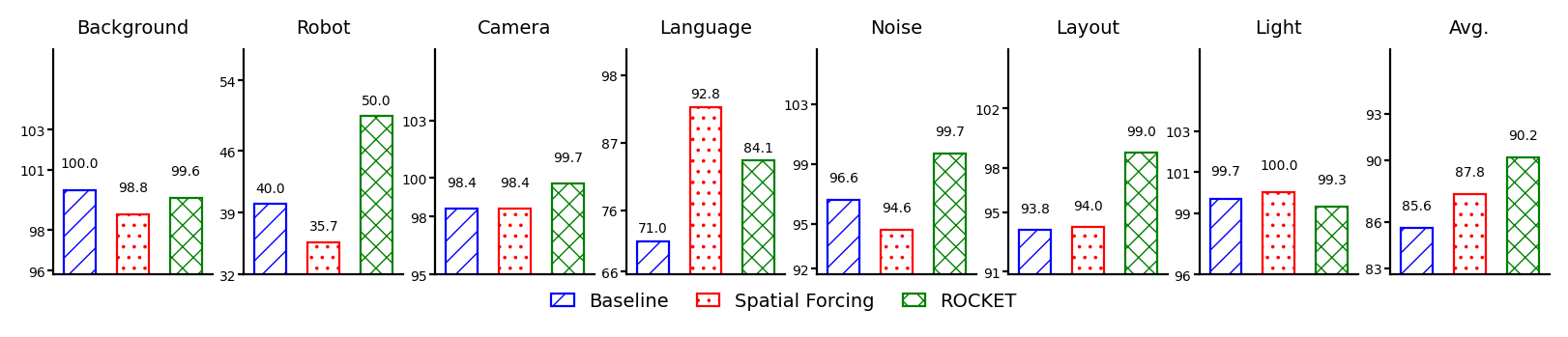}
  \vspace{-0.4cm}
  \caption{Per-task success rate (\%) on LIBERO-Plus \textbf{Spatial} tasks under the seven perturbations.}
  \label{fig:libero_plus_spatial}
\end{figure*}

\begin{figure*}[htbp]
  \centering
  \includegraphics[width=1.01\linewidth]{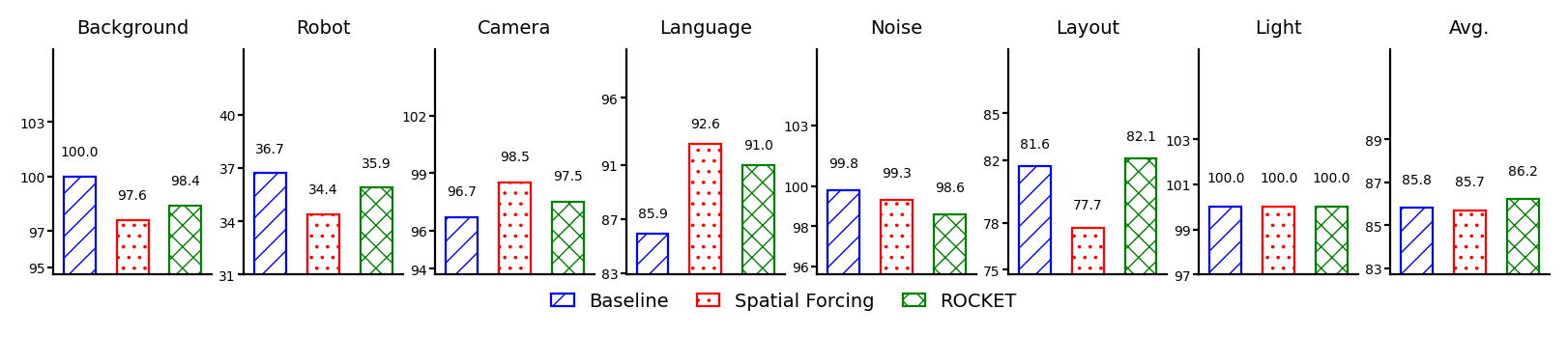}
  \vspace{-0.4cm}
  \caption{Per-task success rate (\%) on LIBERO-Plus \textbf{Object} tasks under the seven perturbations.}
  \label{fig:libero_plus_object}
\end{figure*}

\begin{figure*}[htbp]
  \centering
  \includegraphics[width=1.01\linewidth]{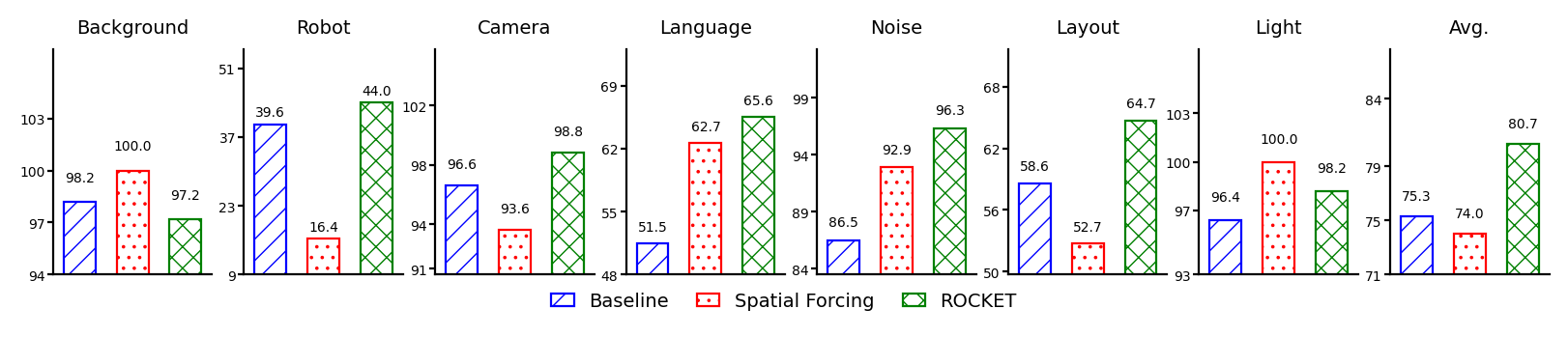}
  \vspace{-0.4cm}
  \caption{Per-task success rate (\%) on LIBERO-Plus \textbf{Goal} tasks under the seven perturbations.}
  \label{fig:libero_plus_goal}
\end{figure*}

\begin{figure*}[htbp]
  \centering
  \includegraphics[width=1.01\linewidth]{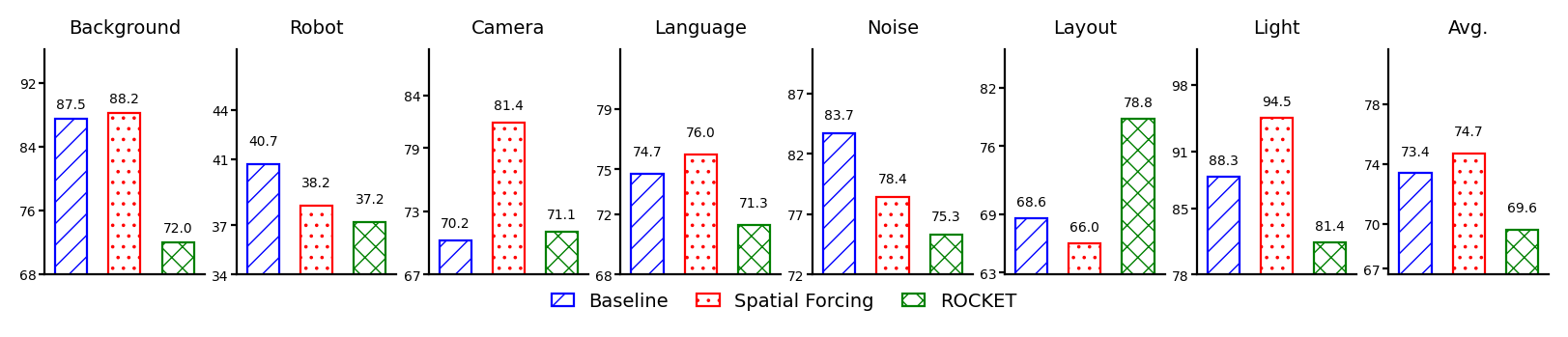}
  \vspace{-0.4cm}
  \caption{Per-task success rate (\%) on LIBERO-Plus \textbf{Long} tasks under the seven perturbations.}
  \label{fig:libero_plus_long}
\end{figure*}

\section{Analysis: Why Shared Projectors Increase Cross-Layer Gradient Coherence}
\label{sec:theory_shared_projector}

\subsection{Takeaway: shared projector makes cross-layer gradients add constructively}

This section explains why a \emph{shared} nonlinear projector tends to yield more \emph{cross-layer coherent}
distillation gradients than \emph{layer-specific} independent projectors in a Pre-LN Transformer.
The argument has two steps:
(i) under Pre-LN \emph{residual-smallness}, back-propagation through the residual stream is close to identity,
so the gradient at an earlier layer is approximately a superposition of \emph{future local} distillation gradients;
(ii) the stability of this superposition is governed by \emph{cross-layer interference} terms, and a shared projector
imposes a common operator structure on these terms (up to controlled deviation), whereas Multiple independent projectors do not.

\subsection{Setup: Pre-LN residual stream, nonlinear projector, and distillation loss}

Let the student network be a Pre-LN Transformer written on the residual stream:
\begin{equation}
\boxed{
h_{l+1} \;=\; h_l \;+\; F_l\!\left(\mathrm{LN}(h_l);\theta_l\right),
\qquad l=0,\dots,L-1
}
\label{eq:preln_residual}
\end{equation}
where $h_l\in\mathbb{R}^d$ denotes the (flattened) residual representation at layer $l$,
$\mathrm{LN}(\cdot)$ is LayerNorm, and $\theta_l$ are the block parameters.
Define $u_l := \mathrm{LN}(h_l)$ and $\Delta_l := F_l(u_l;\theta_l)$ so that
$h_{l+1}=h_l+\Delta_l$.

\paragraph{Nonlinear projector.}
For each distillation layer $l\in\mathcal{S}$, we attach a two-layer MLP projector
(with input normalization $\mathrm{LN}_v$):
\begin{align}
\tilde h_l &:= \mathrm{LN}_v(h_l)
\label{eq:vlm_norm_def}\\
p_{\phi}(h_l) &:= W_2\,\sigma\!\left(W_1 \tilde h_l\right),
\qquad \sigma = \mathrm{GELU},
\label{eq:projector_def}
\end{align}
where $\phi=\{W_1,W_2\}$ (and potentially $\mathrm{LN}_v$ parameters).

\paragraph{Layerwise distillation loss.}
Let $t_l$ denote the teacher representation used as the distillation target for the student's layer $l$, selected via a fixed layer mapping.
We define the multi-layer distillation objective:
\begin{equation}
\mathcal{L}_{\mathrm{align}}(\theta,\{\phi_l\})
\;:=\;
\sum_{l\in\mathcal{S}}\alpha_l\,\ell\!\left(z_l, t_l\right),
\qquad
z_l := p_{\phi_l}(h_l(\theta)),
\label{eq:dist_loss}
\end{equation}
where $\ell(\cdot,\cdot)$ is any differentiable similarity loss
(e.g., $1-\cos(\cdot,\cdot)$), and $\alpha_l \ge 0$ is the coefficient of the alignment loss; by default, we use uniform (mean) weighting.

\subsection{Jacobians needed for gradient transport and projector-induced coupling}

\paragraph{Pre-LN residual Jacobian.}
Differentiating \eqref{eq:preln_residual} yields
\begin{equation}
\frac{\partial h_{l+1}}{\partial h_l}
\;=\;
I + A_l,
\qquad
A_l
:=
\frac{\partial F_l(u_l;\theta_l)}{\partial u_l}\,
J_{\mathrm{LN}}(h_l),
\label{eq:preln_jacobian}
\end{equation}
where $J_{\mathrm{LN}}(h_l)$ denotes the Jacobian of $\mathrm{LN}$ at $h_l$.
Thus, for $i<l$,
\begin{equation}
\frac{\partial h_l}{\partial h_i}
\;=\;
\prod_{k=1}^{l-i} (I + A_{l-k}).
\label{eq:chain_preln}
\end{equation}

\paragraph{Projector Jacobian.}
Let $x := W_1\tilde h$ and $D_{\sigma}(x):=\mathrm{diag}(\sigma'(x))$.
Then the Jacobian of \eqref{eq:projector_def} is
\begin{equation}
J_{p_\phi}(h)
=
W_2\,D_{\sigma}(W_1\tilde h)\,W_1\,J_{\mathrm{LN}_v}(h),
\label{eq:projector_jacobian}
\end{equation}
where $J_{\mathrm{LN}_v}(h)$ is the Jacobian of $\mathrm{LN}_v$.

\subsection{Exact gradient at an early layer and its superposition approximation}
This subsection writes the exact distillation gradient at $h_i$ and then states the approximation that removes
deep transport effects under Pre-LN residual-smallness.

For each $l\in\mathcal{S}$, define the local gradient from the distillation loss:
\begin{equation}
g_l \;:=\; \nabla_{z_l}\,\ell(z_l,t_l).
\label{eq:g_def}
\end{equation}
By the chain rule, the gradient at the student layer $l$ is:
\begin{equation}
\nabla_{h_l}\,\ell(z_l,t_l)
\;=\;
J_{p_{\phi_l}}(h_l)^\top\, g_l.
\label{eq:grad_hl}
\end{equation}
For an earlier layer $i<l$, the gradient is back-propagated through the residual blocks. Note that for a product of matrices $M_{l-1}\dots M_i$, the transpose reverses the order to $M_i^\top \dots M_{l-1}^\top$. Thus, we have:
\begin{equation}
\nabla_{h_i}\,\ell(z_l,t_l)
\;=\;
\left(\prod_{k=i}^{l-1}(I + A_k)^\top\right)
J_{p_{\phi_l}}(h_l)^\top\, g_l,
\label{eq:grad_hi_exact}
\end{equation}
where the product is ordered from $k=i$ (leftmost) to $l-1$ (rightmost).
Summing over all distillation layers $l \ge i$ gives the total gradient at layer $i$:
\begin{equation}
\nabla_{h_i}\,\mathcal{L}_{\mathrm{align}}
\;=\;
\sum_{l\in\mathcal{S}:l\ge i}
\alpha_l
\left(\prod_{k=i}^{l-1}(I + A_k)^\top\right)
J_{p_{\phi_l}}(h_l)^\top\, g_l.
\label{eq:grad_total_exact}
\end{equation}

\subsection{Residual-smallness implies near-identity transport}

\paragraph{Assumption (Residual-smallness for Pre-LN).}
There exists $\varepsilon>0$ such that for all $k$ of interest,
\begin{equation}
\|A_k\|
=
\left\|
\frac{\partial F_k(u_k;\theta_k)}{\partial u_k}\,
J_{\mathrm{LN}}(h_k)
\right\|
\le \varepsilon.
\label{eq:assumption_residual_small}
\end{equation}

\paragraph{Lemma 1 (Near-identity transport).}
Under \eqref{eq:assumption_residual_small}, for any $i<l$, the back-propagation matrix remains close to the identity:
\begin{equation}
\left\|
\prod_{k=i}^{l-1}(I + A_k)^\top - I
\right\|
\le
\exp\big((l-i)\varepsilon\big)-1.
\label{eq:near_identity_lemma}
\end{equation}
For small depth or residuals (i.e., $(l-i)\varepsilon \ll 1$), this bound is approximately $(l-i)\varepsilon$.

\paragraph{Corollary 1 (Gradient is a sum of local distillation gradients, up to transport error).}
Substituting the bound from Lemma 1 into \eqref{eq:grad_total_exact}, we obtain that the gradient at layer $i$ is the superposition of future local gradients plus an error term dependent on depth:
\begin{equation}
\boxed{
\nabla_{h_i}\,\mathcal{L}_{\mathrm{align}}
\;=\;
\sum_{l\in\mathcal{S}:l\ge i}
\alpha_l\,J_{p_{\phi_l}}(h_l)^\top\, g_l
\;+\; \mathcal{E}_i
}
\label{eq:grad_superposition}
\end{equation}
where the error $\mathcal{E}_i$ is bounded by order $O\!\left(\varepsilon\sum_{l\ge i} (l-i)\,\alpha_l\|J_{p_{\phi_l}}^\top g_l\|\right)$.

\subsection{Shared projector makes cross-layer interference constructive}
\label{sec:shared_projector_corrected}

By Corollary~1, the stability of the early-layer update is controlled by how the terms
$J_{p_{\phi_l}}(h_l)^\top g_l$ interact across distillation layers. We now show that sharing the projector
creates a \emph{common operator} governing these interactions (up to controlled deviation), yielding a
signal-aligned lower bound on cross-layer interference.

We compare two parameterization strategies for the auxiliary tasks:
\begin{itemize}
\item \textbf{Multiple independent projectors:} each $l\in\mathcal{S}$ has distinct parameters $\phi_l$.
\item \textbf{Shared projector:} a single set of parameters $\phi_l \equiv \phi$ is used for all $l\in\mathcal{S}$.
\end{itemize}

\paragraph{Gradient forms and the interference term.}
Recalling \eqref{eq:grad_superposition}, the gradient at representation $h_i$ is a superposition of terms.
Let
\begin{equation}
v_l \;:=\; J_{p_{\phi_l}}(h_l)^\top g_l,
\qquad
G \;\approx\; \sum_{l\in\mathcal{S}:l\ge i}\alpha_l\, v_l.
\label{eq:vdef_Gdef}
\end{equation}
The squared norm expands as
\begin{equation}
\|G\|_2^2
\;=\;
\sum_{l}\alpha_l^2\|v_l\|_2^2
\;+\;
\sum_{a\neq b}\alpha_a\alpha_b\,\langle v_a, v_b\rangle,
\label{eq:Gnorm_expand}
\end{equation}
where the cross terms $\langle v_a,v_b\rangle$ determine whether different distillation layers reinforce or
cancel each other.

\paragraph{Proposition 1 (Cross-layer interference as a bilinear form on teacher-side error signals).}
For any $a,b\in\mathcal{S}$,
\begin{equation}
\boxed{
\langle v_a, v_b \rangle
\;=\;
g_a^\top
\underbrace{\left( J_{p_{\phi_a}}(h_a) \, J_{p_{\phi_b}}(h_b)^\top \right)}_{=: \mathcal{M}_{ab}}
g_b
}
\label{eq:interference_def_corrected}
\end{equation}

\paragraph{Error-signal subspace.}
Let the \emph{error-signal subspace} be
\begin{equation}
\mathcal{G}
\;:=\;
\mathrm{span}\{g_l:\ l\in\mathcal{S}\}
\subseteq \mathbb{R}^{\dim(z)}.
\label{eq:error_subspace}
\end{equation}
All teacher-side signals that affect the student update through distillation lie in $\mathcal{G}$ by construction.
Accordingly, we will only require geometric regularity of the shared projector on $\mathcal{G}$.

\subsubsection{\textbf{Case 1 (Shared projector): a common PSD operator governs cross-layer interference.}}

The advantage of shared projector is that cross-layer interactions are governed by (approximately) a \emph{single} operator
on the teacher-side error signals. The following assumptions control how the shared projector Jacobian varies along
the residual stream, enabling us to relate $\mathcal{M}_{ab}$ to a fixed PSD reference operator.

\paragraph{Assumption (Bounded residual increments).}
There exists $\delta>0$ such that for all layers $l$,
\begin{equation}
\|h_{l+1}-h_l\|_2 = \|\Delta_l\|_2 \le \delta.
\label{eq:assumption_delta}
\end{equation}

\paragraph{Assumption (Projector Jacobian Lipschitz).}
There exists $L_J>0$ such that for all $x,y$ in the bounded feature domain,
\begin{equation}
\|J_{p_{\phi}}(x)-J_{p_{\phi}}(y)\|_2 \le L_J\|x-y\|_2.
\label{eq:assumption_LJ}
\end{equation}

\paragraph{Lemma 2 (Cross-layer Jacobian closeness under residual dynamics).}
Under \eqref{eq:assumption_delta} and \eqref{eq:assumption_LJ}, for any $a<b$,
\begin{equation}
\|J_{p_{\phi}}(h_b)-J_{p_{\phi}}(h_a)\|_2
\le
L_J\,(b-a)\,\delta.
\label{eq:jacobian_closeness}
\end{equation}

To convert Jacobian closeness into a statement about alignment, we introduce an \emph{on-subspace}
near-isometry requirement on the teacher-side signals that actually occur.

\paragraph{Assumption (Near-isometry on the error-signal subspace).}
For a reference layer $\bar l\in\mathcal{S}$ and $\bar h := h_{\bar l}$, define
\begin{equation}
J \;:=\; J_{p_\phi}(\bar h),
\qquad
M \;:=\; J J^\top \succeq 0.
\label{eq:JJt_def}
\end{equation}
Assume there exist constants $c>0$ and $\eta\in[0,1)$ such that for all $x\in\mathcal{G}$,
\begin{equation}
(1-\eta)c\,\|x\|_2^2
\;\le\;
x^\top M x
\;\le\;
(1+\eta)c\,\|x\|_2^2.
\label{eq:subspace_isometry}
\end{equation}
Equivalently, $M$ has condition number at most $\frac{1+\eta}{1-\eta}$ when restricted to $\mathcal{G}$.

\paragraph{Lemma 3 (Approximate inner-product preservation on $\mathcal{G}$).}
Under \eqref{eq:subspace_isometry}, for all $x,y\in\mathcal{G}$,
\begin{equation}
\big|x^\top M y - c\, x^\top y\big|
\;\le\;
\eta c \cdot \frac{\|x\|_2^2+\|y\|_2^2}{2}.
\label{eq:innerprod_preserve}
\end{equation}
\textit{Proof sketch.}
Use the polarization identity
$x^\top M y=\frac14\big((x+y)^\top M(x+y)-(x-y)^\top M(x-y)\big)$
and apply the quadratic bounds \eqref{eq:subspace_isometry} to $x\pm y$.

\paragraph{Proposition 2 (Shared projector cross terms decompose into a common PSD part plus deviation).}
Let $\bar h=h_{\bar l}$ and decompose
\begin{equation}
J_{p_\phi}(h_l) = J + E_l,
\qquad
\|E_l\|_2 \le L_J|l-\bar l|\delta,
\label{eq:J_decompose}
\end{equation}
where the bound follows from Lemma~2. Then for any $a,b\in\mathcal{S}$ with $a,b\ge i$,
\begin{equation}
\langle v_a, v_b\rangle_{\mathrm{share}}
=
g_a^\top \big(J_{p_\phi}(h_a)J_{p_\phi}(h_b)^\top\big) g_b
=
g_a^\top M g_b
\;+\; \Delta_{ab},
\label{eq:share_innerprod_decomp}
\end{equation}
where the deviation satisfies
\begin{equation}
|\Delta_{ab}|
\;\le\;
\|g_a\|_2\|g_b\|_2\Big(\|E_a\|_2\|J\|_2 + \|E_b\|_2\|J\|_2 + \|E_a\|_2\|E_b\|_2\Big).
\label{eq:share_deviation_bound}
\end{equation}
Moreover, if $g_a,g_b\in\mathcal{G}$, then by Lemma~3,
\begin{equation}
\boxed{
\langle v_a, v_b\rangle_{\mathrm{share}}
\;\ge\;
c\, g_a^\top g_b
\;-\;
\eta c\cdot \frac{\|g_a\|_2^2+\|g_b\|_2^2}{2}
\;-\;
|\Delta_{ab}|
}
\label{eq:share_lowerbound}
\end{equation}
\textit{Interpretation.}
When (i) the teacher error signals are consistent across layers ($g_a^\top g_b$ positive and not too small),
(ii) the shared projector is near-isometric on $\mathcal{G}$ (small $\eta$),
and (iii) adjacent-layer Jacobians are close (small $\|E_l\|$, i.e., small $(|l-\bar l|\delta)$),
the cross-layer interference term is biased toward being non-destructive and often constructive.

\begin{theorem}[Shared projector induces coherent cross-layer interference under Pre-LN transport]
\label{thm:shared_head_coherence}
Consider the Pre-LN student \eqref{eq:preln_residual} with multi-layer distillation \eqref{eq:dist_loss}.
Under residual-smallness \eqref{eq:assumption_residual_small}, the gradient at an earlier representation $h_i$
admits the superposition form \eqref{eq:grad_superposition}.
For the shared projector parameterization ($\phi_l\equiv\phi$), assume bounded residual increments
\eqref{eq:assumption_delta}, projector Jacobian Lipschitzness \eqref{eq:assumption_LJ}, and near-isometry on
the error-signal subspace \eqref{eq:subspace_isometry}.
Then for any $a,b\in\mathcal{S}$ with $a,b\ge i$ and $g_a,g_b\in\mathcal{G}$, the cross-layer interference
term $\langle v_a,v_b\rangle$ in \eqref{eq:Gnorm_expand} satisfies the shared projector decomposition
\eqref{eq:share_innerprod_decomp} with deviation bounded by \eqref{eq:share_deviation_bound}, and furthermore
admits the signal-aligned lower bound \eqref{eq:share_lowerbound}.
Consequently, when teacher-side error signals are positively correlated across layers (large $g_a^\top g_b$) and the
controlled-deviation terms are small, shared projector gradients exhibit reduced destructive interference and higher
cross-layer coherence in the superposition \eqref{eq:vdef_Gdef}.
\end{theorem}

\subsubsection{\textbf{Case 2 (Multiple independent projectors): lack of a uniform constructive lower bound.}}
For Multiple independent projectors, the cross term becomes
\begin{equation}
\langle v_a, v_b\rangle_{\mathrm{multi}}
=
g_a^\top \big(J_{p_{\phi_a}}(h_a)J_{p_{\phi_b}}(h_b)^\top\big) g_b,
\label{eq:multi_innerprod}
\end{equation}
which does not reduce to a single PSD form $g_a^\top (J J^\top) g_b$ and therefore admits no analogue of
\eqref{eq:innerprod_preserve} or \eqref{eq:share_lowerbound} without additional coupling assumptions.
A generic upper bound is
\begin{equation}
\boxed{
|\langle v_a, v_b\rangle_{\mathrm{multi}}|
\;\le\;
\|g_a\|_2\|g_b\|_2\,
\|J_{p_{\phi_a}}(h_a)\|_2\,
\|J_{p_{\phi_b}}(h_b)\|_2
}
\label{eq:multi_trivial_bound}
\end{equation}
but this bound does not preclude sign flips (destructive interference) since the bilinear form
$J_{p_{\phi_a}}(h_a)J_{p_{\phi_b}}(h_b)^\top$ is not constrained to be PSD or even symmetric.

\textit{Empirical specialization and orthogonalization.}
In our experiments, we observe a consistent ``decoupling'' behavior: as training proceeds and the per-layer alignment losses saturate, different projectors tend to rotate their effective Jacobian row spaces into near-orthogonal subspaces. Concretely, if we measure the pairwise similarity between projector weights (or
their local linearizations), e.g., $\cos\!\left(\mathrm{vec}(W_a),\mathrm{vec}(W_b)\right)$ for linear heads
or an analogous metric based on $J_{p_{\phi_l}}(h_l)$, the values decrease toward (and remain close to) zero in the early stage after training has stabilized, and this behavior persists consistently until the end of training (The visualization results are provided in Appendix~\ref{app:projector_similarity} (Fig.\ref{fig:projector_cosine_similarity}).). 
This indicates that, rather than enforcing a shared representation geometry
across layers, independent heads can absorb layer-specific discrepancies by specializing into mutually
orthogonal directions, which suppresses cross-layer coupling in \eqref{eq:multi_innerprod}.
In Fig.~\ref{fig:grad} (top), we show the cosine similarity between the gradients induced by different alignment losses at an early layer (the projector between the vision tower and the LLM backbone) during the early stage of training and the mid-to-late stage. 
For the $10$ layers, we compute the pairwise cosine similarities among their gradients. The visualization shows that the gradients in the early stage are nearly perfectly orthogonal, indicating substantial training fluctuations; in the mid-to-late stage, this low cosine similarity persists.
In Fig.~\ref{fig:grad} (bottom), we visualize the corresponding results for \OURS. From the very early stage of training to the mid-to-late stage, the shared-projector mechanism in \OURS keeps the gradients induced by different alignment losses consistently aligned. This directional consistency holds throughout training, and the resulting mutual reinforcement among gradients contributes to the superior performance of \OURS. 
Specifically, Fig.~\ref{fig:grad_sim_early} shows the gradient similarities in the early stage of training, while Fig.~\ref{fig:grad_sim_late} shows the gradient similarities in the mid-to-late stage.

As a consequence, even when the teacher error signals are correlated across layers (i.e., $g_a^\top g_b>0$),
the induced student-side gradients $v_a$ and $v_b$ can become weakly correlated because the interaction matrix
$J_{p_{\phi_a}}(h_a)J_{p_{\phi_b}}(h_b)^\top$ itself shrinks in the relevant directions. In practice, this
often yields a small average cross term $\mathbb{E}[\langle v_a,v_b\rangle_{\mathrm{multi}}]$ and can even
lead to destructive interference when the effective interaction changes sign across mini-batches.

\paragraph{Consequence for gradient stability.}
Plugging \eqref{eq:share_lowerbound} (shared projector) into \eqref{eq:Gnorm_expand} shows that the second summation
in \eqref{eq:Gnorm_expand} admits a \emph{signal-aligned} lower bound (up to controlled errors),
while Multiple independent projectors admit only the trivial magnitude bound \eqref{eq:multi_trivial_bound}.
Thus, under the same transport approximation of Corollary~1, shared projector gradients
tend to exhibit higher cross-layer coherence and reduced destructive interference, consistent with empirical observations.

\clearpage

\begin{figure}[htbp]
    \centering
    \includegraphics[width=\linewidth]{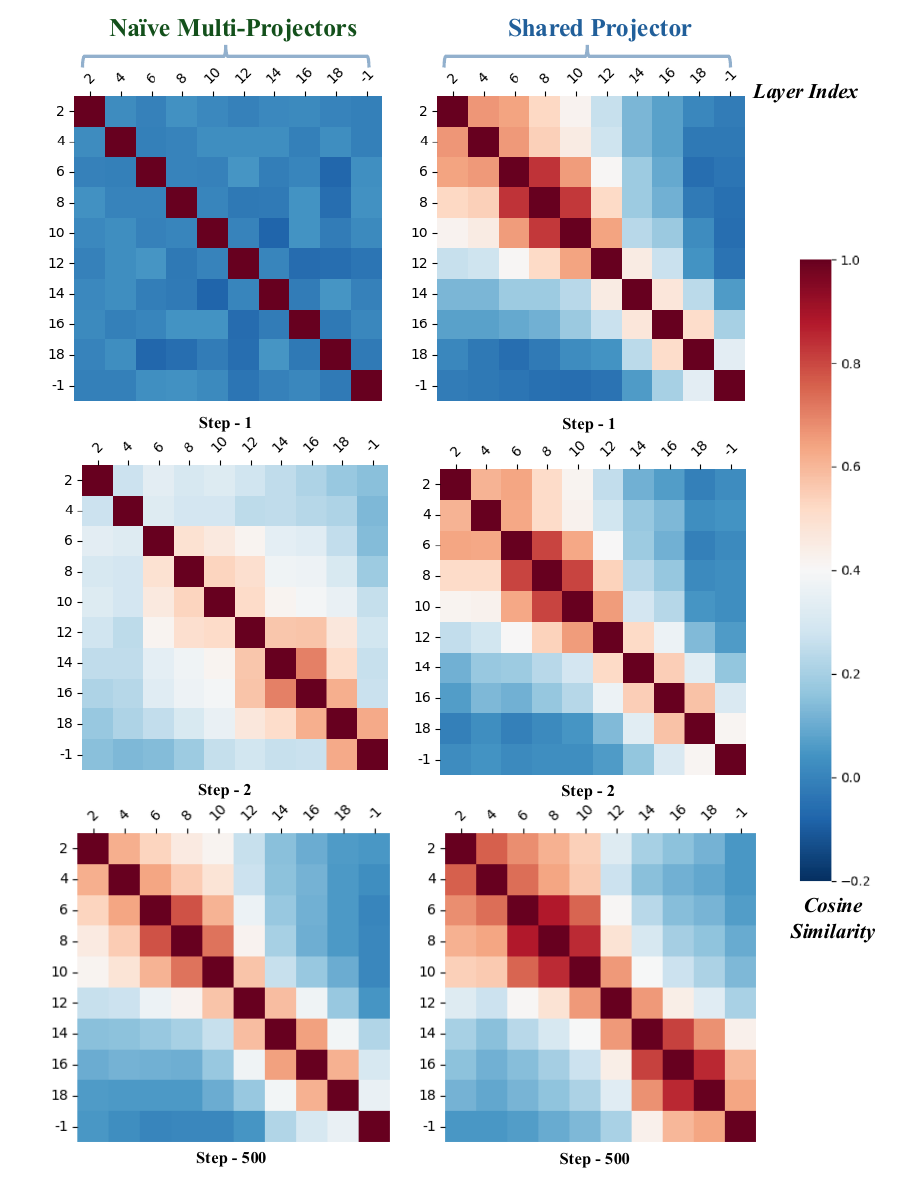}
    \caption{Gradient cosine similarity in the early stage of training.}
    \label{fig:grad_sim_early}
\end{figure}

\begin{figure}[htbp]
    \centering
    \includegraphics[width=\linewidth]{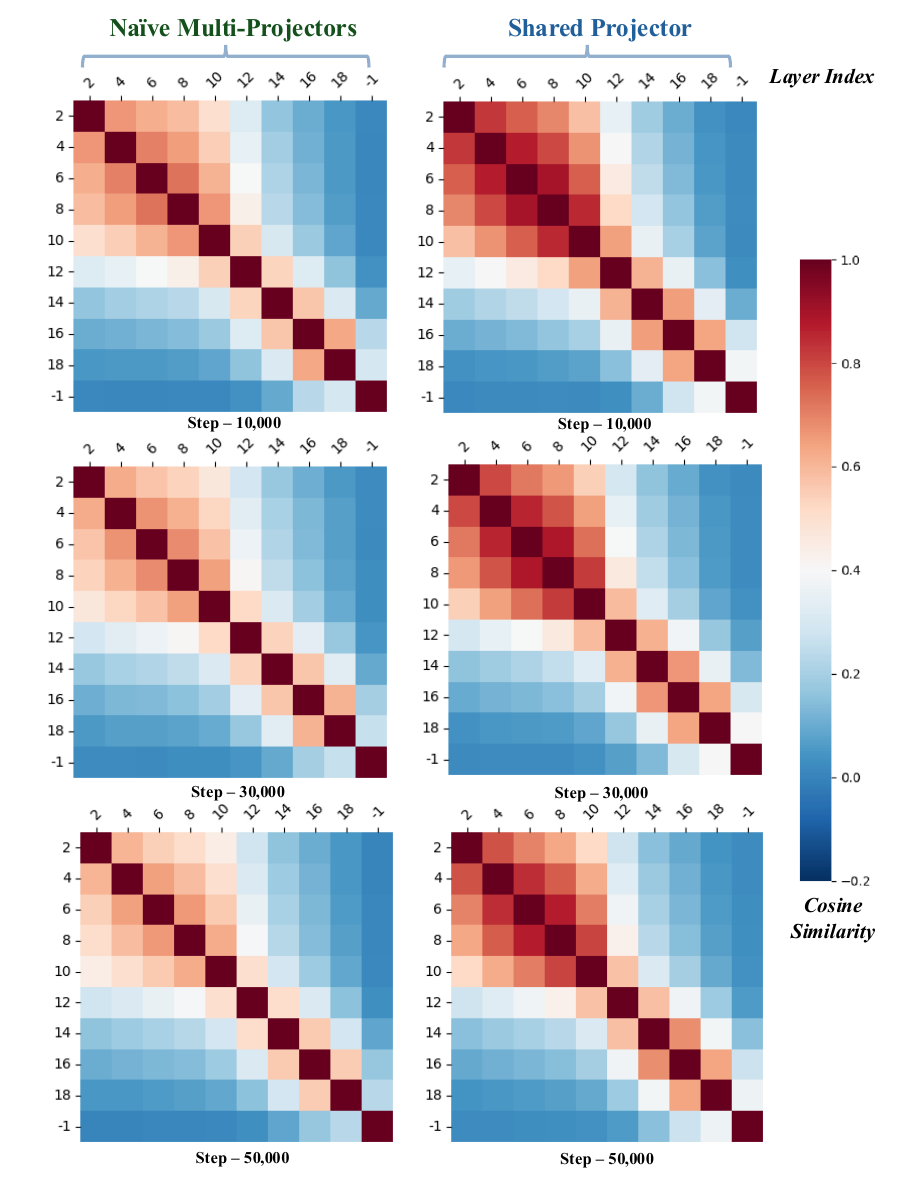}
    \caption{Gradient cosine similarity in the mid-to-late stage of training.}
    \label{fig:grad_sim_late}
\end{figure}

\clearpage

\section{Details of Layer Selection Strategies}

\subsection{Similarity-based Layer Selection}
\label{app:layer_selection_strategy}

\begin{figure*}[htbp]
    \centering
    \includegraphics[width=0.85\textwidth]{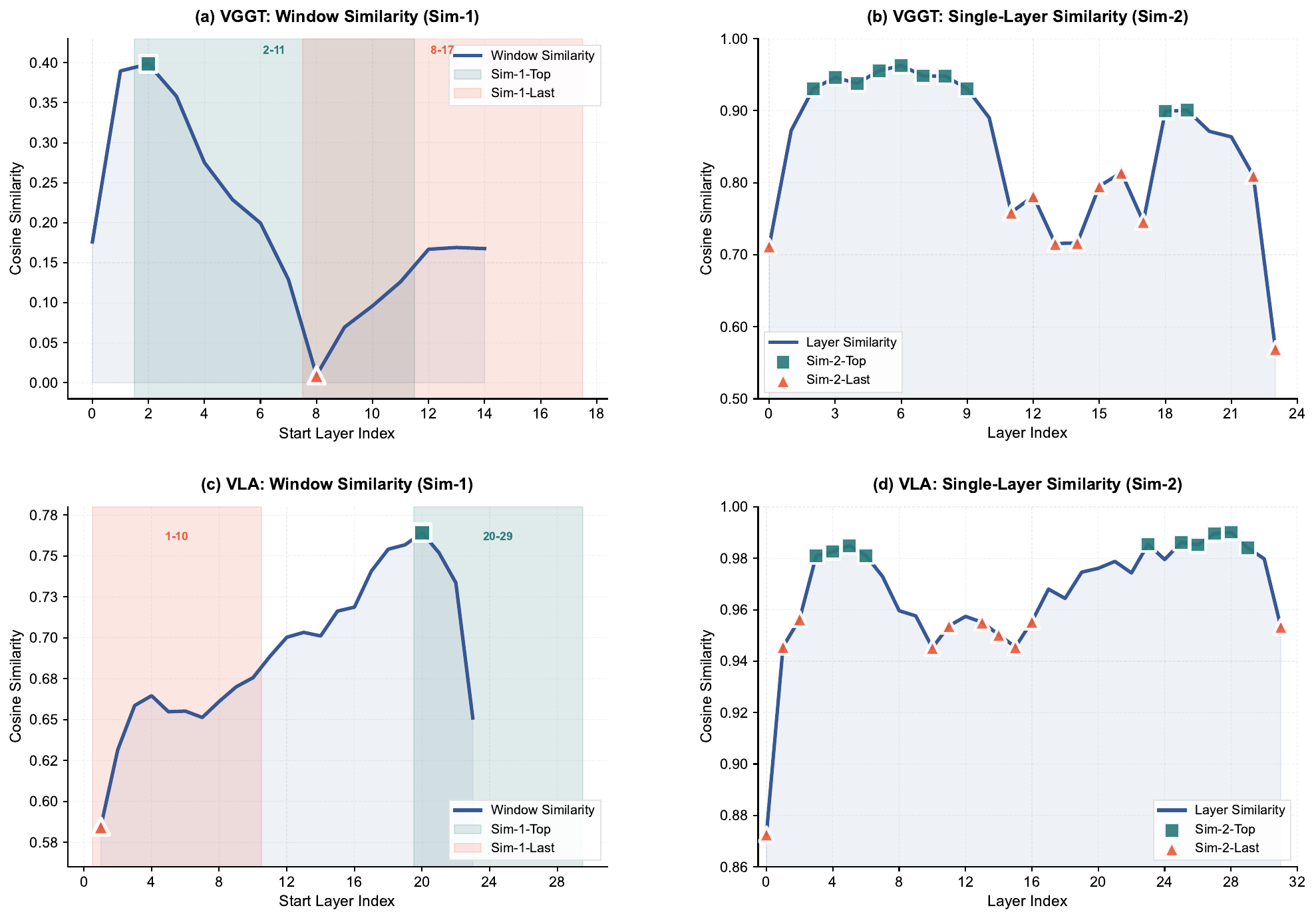} 
    \caption{
    \textbf{Similarity-based layer selection strategies.} 
    We evaluate two criteria for selecting alignment layers: 
    \textbf{Sim-1} (left) measures 10-layer window similarity by computing cosine similarity between layer $i$'s input and layer $i{+}9$'s output; 
    \textbf{Sim-2} (right) measures single-layer similarity between each layer's input and output. 
    For each criterion, we select 10 layers with either highest similarity (\textbf{Top}, green squares) or lowest similarity (\textbf{Last}, orange triangles). 
    }
    \label{fig:layer_select}
\end{figure*}

Figure~\ref{fig:layer_select} shows that VGGT (24 layers, top row) and VLA (32 layers, bottom row) reveal distinct similarity distributions across the network depth. Lower similarity indicates greater information transformation, potentially identifying more informative layers for alignment.

\newpage
\subsection{Detailed Chosen indices for All Layer Selection Strategies }
\label{app:details_index}
\input{tables/layer_indice_detail}

%% file: tables/layer_indice_detail.tex
\begin{table*}[htbp]
\centering
\label{tab:layer_selection_strategies}
\small
\setlength{\tabcolsep}{6pt}
\begin{tabular}{l|c|p{6.5cm}|p{6.5cm}}
\toprule
\textbf{Strategy} & \textbf{\# Layers} & \textbf{VLA Layers} & \textbf{VGGT Layers} \\
\midrule
\multicolumn{4}{c}{\textit{Uniform Sampling}} \\
\midrule
Uniform-4  & 4  & 8, 16, 24, 32 & 5, 11, 17, 23 \\
\addlinespace[2pt]
Uniform-8  & 8  & 4, 8, 12, 16, 20, 24, 28, 32 & 2, 5, 8, 11, 14, 17, 20, 23 \\
\addlinespace[2pt]
Uniform-16 & 16 & 4, 5, 6, 7, 20, 21, 22, 23, 24, 25, 26, 27, 28, 29, 30, 31 & 1, 2, 4, 5, 7, 8, 10, 11, 13, 14, 16, 17, 19, 20, 22, 23 \\
\addlinespace[2pt]
Uniform-24 & 24 & 2, 3, 4, 6, 7, 8, 10, 11, 12, 14, 15, 16, 18, 19, 20, 22, 23, 24, 26, 27, 28, 30, 31, 32 & 0, 1, 2, 3, 4, 5, 6, 7, 8, 9, 10, 11, 12, 13, 14, 15, 16, 17, 18, 19, 20, 21, 22, 23 \\
\midrule
\multicolumn{4}{c}{\textit{Depth-based Sampling}} \\
\midrule
Shallow & 10 & 1, 2, 3, 4, 5, 6, 7, 8, 9, 10 & 0, 1, 2, 3, 4, 5, 6, 7, 8, 9 \\
\addlinespace[2pt]
Middle  & 10 & 12, 13, 14, 15, 16, 17, 18, 19, 20, 21 & 7, 8, 9, 10, 11, 12, 13, 14, 15, 16 \\
\addlinespace[2pt]
Deep    & 10 & 23, 24, 25, 26, 27, 28, 29, 30, 31, 32 & 14, 15, 16, 17, 18, 19, 20, 21, 22, 23 \\
\midrule
\multicolumn{4}{c}{\textit{Even-to-Max (E2M) Sampling}} \\
\midrule
\(E2M\text{-}Last1^{*}\) & 10 & 2, 4, 6, 8, 10, 12, 14, 16, 18, 33 & 2, 4, 6, 8, 10, 12, 14, 16, 18, 23 \\
\addlinespace[2pt]
\(E2M\text{-}Last1^{\dagger}\) & 10 & 2, 3, 4, 5, 6, 7, 8, 9, 10, 18 & 2, 4, 6, 8, 10, 12, 14, 16, 18, 23 \\
\midrule
\multicolumn{4}{c}{\textit{Similarity-based Sampling}} \\
\midrule
Sim-1-Top  & 10 & 20, 21, 22, 23, 24, 25, 26, 27, 28, 29 & 2, 3, 4, 5, 6, 7, 8, 9, 10, 11 \\
\addlinespace[2pt]
Sim-1-Last & 10 & 1, 2, 3, 4, 5, 6, 7, 8, 9, 10 & 8, 9, 10, 11, 12, 13, 14, 15, 16, 17 \\
\addlinespace[2pt]
Sim-2-Top  & 10 & 4, 5, 6, 7, 24, 26, 27, 28, 29, 30 & 2, 3, 4, 5, 6, 7, 8, 9, 18, 19 \\
\addlinespace[2pt]
Sim-2-Last & 10 & 1, 2, 3, 11, 12, 14, 15, 16, 17, 32 & 0, 11, 12, 13, 14, 15, 16, 17, 22, 23 \\
\bottomrule
\end{tabular}
\caption{Detailed layer selection for multi-layer alignment. We show the aligned layer indices for both the student VLA visual backbone and the teacher VGGT model. For OpenVLA, we count the projector between vision tower and LLM as a layer so there are total 33 indices in our default setting. The selected layers are sorted in ascending order for both VLA and VGGT.
$^{*}$ For OpenVLA (32 layers, or 33 including the projector), we select layers uniformly using indices 2--18 plus the last layer. 
$^{\dagger}$ For PI0 and PI0.5 (18 layers), since layer 18 coincides with the last layer, we instead select indices 2--10 plus the last layer.}
\end{table*}

%% file: references.bib
@article{kim2024openvla,
  title   = {OpenVLA: An Open-Source Vision-Language-Action Model},
  author  = {Kim, Moo Jin and Pertsch, Karl and Karamcheti, Siddharth and Xiao, Ted and Balakrishna, Ashwin and Nair, Suraj and Rafailov, Rafael and Foster, Ethan and Lam, Grace and Sanketi, Pannag and Vuong, Quan and Kollar, Thomas and Burchfiel, Benjamin and Tedrake, Russ and Sadigh, Dorsa and Levine, Sergey and Liang, Percy and Finn, Chelsea},
  journal = {arXiv preprint arXiv:2406.09246},
  year    = {2024}
}

@article{Peebles2022DiT,
  title={Scalable Diffusion Models with Transformers},
  author={William Peebles and Saining Xie},
  year={2022},
  journal={arXiv preprint arXiv:2212.09748},
}

@InProceedings{pmlr-v229-zitkovich23a,
  title     = {RT-2: Vision-Language-Action Models Transfer Web Knowledge to Robotic Control},
  author    = {Zitkovich, Brianna and Joshi, Aviral and Brohan, Anthony and Brown, Noah and Chebotar, Yevgen and Cortes, Oscar and Ding, Tianli and Driess, Danny and Dubey, Avinava Kumar and Florence, Pete and Fu, Chuyuan and Gopalakrishnan, Keerthana and Han, Kehang and Hausman, Karol and Herzog, Alexander and Hsu, Jasmine and Ichter, Brian and Irpan, Alex and Kollar, Thomas and Lam, Grace and Levine, Sergey and Liang, Percy and Nair, Suraj and Narang, Yashraj and Pathak, Deepak and Radosavovic, Ilija and Raff, Edward and Rao, Kanishka and Rettinghouse, Zachary and Salazar, Julian and Sermanet, Pierre and Sievers, Tobias and Tan, Jie and Toussaint, Marc and Vanhoucke, Vincent and Vuong, Quan and Wulfmeier, Markus and Xiao, Ted and Finn, Chelsea and Guadarrama, Sergio},
  booktitle = {Proceedings of The 7th Conference on Robot Learning},
  pages     = {2165--2183},
  year      = {2023},
  editor    = {Tan, Jie and Toussaint, Marc and Darvish, Kourosh},
  volume    = {229},
  series    = {Proceedings of Machine Learning Research},
  month     = {06--09 Nov},
  publisher = {PMLR},
  url       = {https://proceedings.mlr.press/v229/zitkovich23a.html}
}

@InProceedings{pmlr-v164-shridhar22a,
  title     = {CLIPort: What and Where Pathways for Robotic Manipulation},
  author    = {Shridhar, Mohit and Manuelli, Lucas and Fox, Dieter},
  booktitle = {Proceedings of The 5th Conference on Robot Learning},
  pages     = {894--906},
  year      = {2022},
  editor    = {Faust, Aleksandra and Hsu, David and Neumann, Gerhard},
  volume    = {164},
  series    = {Proceedings of Machine Learning Research},
  month     = {08--11 Nov},
  publisher = {PMLR},
  url       = {https://proceedings.mlr.press/v164/shridhar22a.html}
}

@InProceedings{pmlr-v205-shridhar23a,
  title     = {Perceiver-Actor: A Multi-Task Transformer for Robotic Manipulation},
  author    = {Shridhar, Mohit and Manuelli, Lucas and Fox, Dieter},
  booktitle = {Proceedings of The 6th Conference on Robot Learning},
  pages     = {785--799},
  year      = {2023},
  editor    = {Liu, Karen and Kulic, Dana and Ichnowski, Jeff},
  volume    = {205},
  series    = {Proceedings of Machine Learning Research},
  month     = {14--18 Dec},
  publisher = {PMLR},
  url       = {https://proceedings.mlr.press/v205/shridhar23a.html}
}

@InProceedings{pmlr-v229-huang23b,
  title     = {VoxPoser: Composable 3D Value Maps for Robotic Manipulation with Language Models},
  author    = {Huang, Wenlong and Wang, Chen and Zhang, Ruohan and Li, Yunzhu and Wu, Jiajun and Fei-Fei, Li},
  booktitle = {Proceedings of The 7th Conference on Robot Learning},
  pages     = {2225--2240},
  year      = {2023},
  editor    = {Tan, Jie and Toussaint, Marc and Darvish, Kourosh},
  volume    = {229},
  series    = {Proceedings of Machine Learning Research},
  month     = {06--09 Nov},
  publisher = {PMLR}
}

@inproceedings{ranftl2021dpt,
  title     = {Vision Transformers for Dense Prediction},
  author    = {Ranftl, Ren{\'e} and Bochkovskiy, Alexey and Koltun, Vladlen},
  booktitle = {Proceedings of the IEEE/CVF International Conference on Computer Vision (ICCV)},
  pages     = {12179--12188},
  year      = {2021}
}

@article{yu2024representation,
  title={Representation alignment for generation: Training diffusion transformers is easier than you think},
  author={Yu, Sihyun and Kwak, Sangkyung and Jang, Huiwon and Jeong, Jongheon and Huang, Jonathan and Shin, Jinwoo and Xie, Saining},
  journal={arXiv preprint arXiv:2410.06940},
  year={2024}
}

@article{wang2024reconstructive,
  title={Reconstructive visual instruction tuning},
  author={Wang, Haochen and Zheng, Anlin and Zhao, Yucheng and Wang, Tiancai and Ge, Zheng and Zhang, Xiangyu and Zhang, Zhaoxiang},
  journal={arXiv preprint arXiv:2410.09575},
  year={2024}
}

@article{huang2025mllms,
  title={MLLMs Need 3D-Aware Representation Supervision for Scene Understanding},
  author={Huang, Xiaohu and Wu, Jingjing and Xie, Qunyi and Han, Kai},
  journal={arXiv preprint arXiv:2506.01946},
  year={2025}
}

@article{li2025spatial,
  title={Spatial Forcing: Implicit Spatial Representation Alignment for Vision-language-action Model},
  author={Li, Fuhao and Song, Wenxuan and Zhao, Han and Wang, Jingbo and Ding, Pengxiang and Wang, Donglin and Zeng, Long and Li, Haoang},
  journal={arXiv preprint arXiv:2510.12276},
  year={2025}
}

@article{lora2021,
  title={LoRA: Low-Rank Adaptation of Large Language Models},
  author={Hu, Edward J and Shen, Yelong and Wallis, Phillip and others},
  journal={arXiv preprint arXiv:2106.09685},
  year={2021}
}

@article{sun2019patient,
  title={Patient knowledge distillation for bert model compression},
  author={Sun, Siqi and Cheng, Yu and Gan, Zhe and Liu, Jingjing},
  journal={arXiv preprint arXiv:1908.09355},
  year={2019}
}

@article{gromov2024unreasonable,
  title={The unreasonable ineffectiveness of the deeper layers},
  author={Gromov, Andrey and Tirumala, Kushal and Shapourian, Hassan and Glorioso, Paolo and Roberts, Daniel A},
  journal={arXiv preprint arXiv:2403.17887},
  year={2024}
}

@article{he2024matters,
  title={What matters in transformers? not all attention is needed},
  author={He, Shwai and Sun, Guoheng and Shen, Zheyu and Li, Ang},
  journal={arXiv preprint arXiv:2406.15786},
  year={2024}
}

@inproceedings{chen2021distilling,
  title={Distilling knowledge via knowledge review},
  author={Chen, Pengguang and Liu, Shu and Zhao, Hengshuang and Jia, Jiaya},
  booktitle={Proceedings of the IEEE/CVF conference on computer vision and pattern recognition},
  pages={5008--5017},
  year={2021}
}

@article{kim2025fine,
  title={Fine-tuning vision-language-action models: Optimizing speed and success},
  author={Kim, Moo Jin and Finn, Chelsea and Liang, Percy},
  journal={arXiv preprint arXiv:2502.19645},
  year={2025}
}

@article{black2024pi_0,
  title={$\pi_0$: A Vision-Language-Action Flow Model for General Robot Control},
  author={Black, Kevin and Brown, Noah and Driess, Danny and Esmail, Adnan and Equi, Michael and Finn, Chelsea and Fusai, Niccolo and Groom, Lachy and Hausman, Karol and Ichter, Brian and Jakubczak, Szymon and Jones, Tim and Ke, Liyiming and Levine, Sergey and Li-Bell, Adrian and Mothukuri, Mohith and Nair, Suraj and Pertsch, Karl and Shi, Lucy Xiaoyang and Tanner, James and Vuong, Quan and Walling, Anna and Wang, Haohuan and Zhilinsky, Ury},
  journal={arXiv preprint arXiv:2410.24164},
  year={2024},
  url={https://arxiv.org/abs/2410.24164}
}

@inproceedings{black2025pi05,
  title={$\pi_{0.5}$: a Vision-Language-Action Model with Open-World Generalization},
  author={Black, Kevin and Brown, Noah and Darpinian, James and Dhabalia, Karan and Driess, Danny and Esmail, Adnan and Equi, Michael and Finn, Chelsea and Fusai, Niccolo and Galliker, Manuel Y. and Ghosh, Dibya and Groom, Lachy and Hausman, Karol and Ichter, Brian and Jakubczak, Szymon and Jones, Tim and Ke, Liyiming and LeBlanc, Devin and Levine, Sergey and Li-Bell, Adrian and Mothukuri, Mohith and Nair, Suraj and Pertsch, Karl and Ren, Allen Z. and Shi, Lucy Xiaoyang and Smith, Laura and Springenberg, Jost Tobias and Stachowicz, Kyle and Tanner, James and Vuong, Quan and Walke, Homer and Walling, Anna and Wang, Haohuan and Yu, Lili and Zhilinsky, Ury},
  booktitle={Proceedings of The 9th Conference on Robot Learning (CoRL)},
  series={Proceedings of Machine Learning Research},
  volume={305},
  pages={17--40},
  year={2025},
  publisher={PMLR}
}

@article{qu2025spatialvla,
  title={Spatialvla: Exploring spatial representations for visual-language-action model},
  author={Qu, Delin and Song, Haoming and Chen, Qizhi and Yao, Yuanqi and Ye, Xinyi and Ding, Yan and Wang, Zhigang and Gu, JiaYuan and Zhao, Bin and Wang, Dong and Li, Xuelong},
  journal={arXiv preprint arXiv:2501.15830},
  year={2025}
}

@article{bhat20253d,
  title={3D CAVLA: Leveraging Depth and 3D Context to Generalize Vision Language Action Models for Unseen Tasks},
  author={Bhat, Vineet and Lan, Yu-Hsiang and Krishnamurthy, Prashanth and Karri, Ramesh and Khorrami, Farshad},
  journal={arXiv preprint arXiv:2505.05800},
  year={2025}
}

@inproceedings{wang2025vggt,
  title={VGGT: Visual geometry grounded transformer},
  author={Wang, Jianyuan and Chen, Minghao and Karaev, Nikita and Vedaldi, Andrea and Rupprecht, Christian and Novotny, David},
  booktitle={Proceedings of the Computer Vision and Pattern Recognition Conference},
  pages={5294--5306},
  year={2025}
}

@article{lin2025depth,
  title={Depth anything 3: Recovering the visual space from any views},
  author={Lin, Haotong and Chen, Sili and Liew, Junhao and Chen, Donny Y and Li, Zhenyu and Shi, Guang and Feng, Jiashi and Kang, Bingyi},
  journal={arXiv preprint arXiv:2511.10647},
  year={2025}
}

@inproceedings{gong-etal-2025-beyond,
    title = "Beyond Logits: Aligning Feature Dynamics for Effective Knowledge Distillation",
    author = "Gong, Guoqiang  and
      Wang, Jiaxing  and
      Xu, Jin  and
      Xiang, Deping  and
      Zhang, Zicheng  and
      Shen, Leqi  and
      Zhang, Yifeng  and
      JunhuaShu, JunhuaShu  and
      ZhaolongXing, ZhaolongXing  and
      Chen, Zhen  and
      Liu, Pengzhang  and
      Zhang, Ke",
    editor = "Che, Wanxiang  and
      Nabende, Joyce  and
      Shutova, Ekaterina  and
      Pilehvar, Mohammad Taher",
    booktitle = "Proceedings of the 63rd Annual Meeting of the Association for Computational Linguistics (Volume 1: Long Papers)",
    month = jul,
    year = "2025",
    address = "Vienna, Austria",
    publisher = "Association for Computational Linguistics",
    url = "https://aclanthology.org/2025.acl-long.1125/",
    doi = "10.18653/v1/2025.acl-long.1125",
    pages = "23067--23077",
    ISBN = "979-8-89176-251-0"
}

@inproceedings{miles2024understanding,
  title={Understanding the role of the projector in knowledge distillation},
  author={Miles, Roy and Mikolajczyk, Krystian},
  booktitle={Proceedings of the AAAI Conference on Artificial Intelligence},
  volume={38},
  number={5},
  pages={4233--4241},
  year={2024}
}

@inproceedings{miles2024vkd,
  title={Vkd: Improving knowledge distillation using orthogonal projections},
  author={Miles, Roy and Elezi, Ismail and Deng, Jiankang},
  booktitle={Proceedings of the IEEE/CVF Conference on Computer Vision and Pattern Recognition},
  pages={15720--15730},
  year={2024}
}

@article{gao2019representation,
  title={Representation degeneration problem in training natural language generation models},
  author={Gao, Jun and He, Di and Tan, Xu and Qin, Tao and Wang, Liwei and Liu, Tie-Yan},
  journal={arXiv preprint arXiv:1907.12009},
  year={2019}
}

@article{sun2025geovla,
  title={Geovla: Empowering 3d representations in vision-language-action models},
  author={Sun, Lin and Xie, Bin and Liu, Yingfei and Shi, Hao and Wang, Tiancai and Cao, Jiale},
  journal={arXiv preprint arXiv:2508.09071},
  year={2025}
}

@article{wang2025pi,
  title={$\pi^{3}$: Permutation-Equivariant Visual Geometry Learning},
  author={Wang, Yifan and Zhou, Jianjun and Zhu, Haoyi and Chang, Wenzheng and Zhou, Yang and Li, Zizun and Chen, Junyi and Pang, Jiangmiao and Shen, Chunhua and He, Tong},
  journal={arXiv preprint arXiv:2507.13347},
  year={2025}
}

@inproceedings{zhai2023sigmoid,
  title={Sigmoid loss for language image pre-training},
  author={Zhai, Xiaohua and Mustafa, Basil and Kolesnikov, Alexander and Beyer, Lucas},
  booktitle={Proceedings of the IEEE/CVF international conference on computer vision},
  pages={11975--11986},
  year={2023}
}

@article{liu2023libero,
  title   = {LIBERO: Benchmarking Knowledge Transfer for Lifelong Robot Learning},
  author  = {Liu, Haoyi and Zhu, Yuke and Anand, Anirudh and Guo, Zichen and Stone, Peter and Levine, Sergey},
  journal = {arXiv preprint arXiv:2306.03310},
  year    = {2023}
}

@article{chen2025robotwin,
  title   = {RoboTwin 2.0: A Scalable Data Generator and Benchmark with Strong Domain Randomization for Robust Bimanual Robotic Manipulation},
  author  = {Chen, Tianxing and Chen, Zanxin and Chen, Baijun and Cai, Zijian and Liu, Yibin and Li, Zixuan and Liang, Qiwei and Lin, Xianliang and Ge, Yiheng and Gu, Zhenyu and others},
  journal = {arXiv preprint arXiv:2506.18088},
  year    = {2025}
}

@article{oquab2023dinov2,
  title={Dinov2: Learning robust visual features without supervision},
  author={Oquab, Maxime and Darcet, Timoth{\'e}e and Moutakanni, Th{\'e}o and Vo, Huy and Szafraniec, Marc and Khalidov, Vasil and Fernandez, Pierre and Haziza, Daniel and Massa, Francisco and El-Nouby, Alaaeldin and others},
  journal={arXiv preprint arXiv:2304.07193},
  year={2023}
}

@inproceedings{karamcheti2024prismatic,
  title={Prismatic vlms: Investigating the design space of visually-conditioned language models},
  author={Karamcheti, Siddharth and Nair, Suraj and Balakrishna, Ashwin and Liang, Percy and Kollar, Thomas and Sadigh, Dorsa},
  booktitle={Forty-first International Conference on Machine Learning},
  year={2024}
}

@article{touvron2023llama,
  title={Llama 2: Open foundation and fine-tuned chat models},
  author={Hugo Touvron and Louis Martin and Kevin Stone and Peter Albert and Amjad Almahairi and Yasmine Babaei and Nikolay Bashlykov and Soumya Batra and Prajjwal Bhargava and Shruti Bhosale and Dan Bikel and Lukas Blecher and Cristian Canton Ferrer and Moya Chen and Guillem Cucurull and David Esiobu and Jude Fernandes and Jeremy Fu and Wenyin Fu and Brian Fuller and Cynthia Gao and Vedanuj Goswami and Naman Goyal and Anthony Hartshorn and Saghar Hosseini and Rui Hou and Hakan Inan and Marcin Kardas and Viktor Kerkez and Madian Khabsa and Isabel Kloumann and Artem Korenev and Punit Singh Koura and Marie-Anne Lachaux and Thibaut Lavril and Jenya Lee and Diana Liskovich and Yinghai Lu and Yuning Mao and Xavier Martinet and Todor Mihaylov and Pushkar Mishra and Igor Molybog and Yixin Nie and Andrew Poulton and Jeremy Reizenstein and Rashi Rungta and Kalyan Saladi and Alan Schelten and Ruan Silva and Eric Michael Smith and Ranjan Subramanian and Xiaoqing Ellen Tan and Binh Tang and Ross Taylor and Adina Williams and Jian Xiang Kuan and Puxin Xu and Zheng Yan and Iliyan Zarov and Yuchen Zhang and Angela Fan and Melanie Kambadur and Sharan Narang and Aurelien Rodriguez and Robert Stojnic and Sergey Edunov and Thomas Scialom},
  journal={arXiv preprint arXiv:2307.09288},
  year={2023}
}

@inproceedings{perez2018film,
  title={FiLM: Visual reasoning with a general conditioning layer},
  author={Perez, Ethan and Strub, Florian and De Vries, Harm and Dumoulin, Vincent and Courville, Aaron},
  booktitle={Proceedings of the AAAI conference on artificial intelligence},
  volume={32},
  number={1},
  year={2018}
}

@article{beyer2024paligemma,
  title={PaliGemma: A versatile 3B VLM for transfer},
  author={Beyer, Lucas and Steiner, Andreas and Pinto, Andr{\'e} Susano and Kolesnikov, Alexander and Wang, Xiao and Salz, Daniel and Neumann, Maxim and Alabdulmohsin, Ibrahim and Tschannen, Michael and Bugliarello, Emanuele and Unterthiner, Thomas and Keysers, Daniel and Koppula, Skanda and Liu, Fangyu and Grycner, Adam and Gritsenko, Alexey and Houlsby, Neil and Kumar, Manoj and Rong, Keran and Eisenschlos, Julian and Kabra, Rishabh and Bauer, Matthias and Bo{\v{s}}njak, Matko and Chen, Xi and Minderer, Matthias and Voigtlaender, Paul and Bica, Ioana and Balazevic, Ivana and Puigcerver, Joan and Papalampidi, Pinelopi and Henaff, Olivier and Xiong, Xi and Soricut, Radu and Harmsen, Jeremiah and Zhai, Xiaohua},
  journal={arXiv preprint arXiv:2407.07726},
  year={2024}
}

@article{team2024gemma,
  title={Gemma: Open models based on gemini research and technology},
  author={Mesnard, Thomas and Hardin, Cassidy and Dadashi, Robert and Bhupatiraju, Surya and Pathak, Shreya and Sifre, Laurent and Rivi{\`e}re, Morgane and Kale, Mihir Sanjay and Love, Juliette and others},
  journal={arXiv preprint arXiv:2403.08295},
  year={2024}
}

@article{chi2025diffusion,
  title={Diffusion policy: Visuomotor policy learning via action diffusion},
  author={Chi, Cheng and Xu, Zhenjia and Feng, Siyuan and Cousineau, Eric and Du, Yilun and Burchfiel, Benjamin and Tedrake, Russ and Song, Shuran},
  journal={The International Journal of Robotics Research},
  volume={44},
  number={10-11},
  pages={1684--1704},
  year={2025},
  publisher={Sage Publications Sage UK: London, England}
}

@article{zheng2024tracevla,
  title={Tracevla: Visual trace prompting enhances spatial-temporal awareness for generalist robotic policies},
  author={Zheng, Ruijie and Liang, Yongyuan and Huang, Shuaiyi and Gao, Jianfeng and Daum{\'e} III, Hal and Kolobov, Andrey and Huang, Furong and Yang, Jianwei},
  journal={arXiv preprint arXiv:2412.10345},
  year={2024}
}

@inproceedings{ghosh2024octo,
  title={Octo: An open-source generalist robot policy},
  author={Ghosh, Dibya and Walke, Homer and Pertsch, Karl and Black, Kevin and Mees, Oier and Dasari, Sudeep and Hejna, Joey and Kreiman, Tobias and Xu, Charles and Luo, Jianlan and Tan, You Liang and Chen, Lawrence Yunliang and Sanketi, Pannag and Vuong, Quan and Xiao, Ted and Sadigh, Dorsa and Finn, Chelsea and Levine, Sergey},
  booktitle={Proceedings of Robotics: Science and Systems (RSS)},
  year={2024}
}

@article{hou2025dita,
  title={Dita: Scaling diffusion transformer for generalist vision-language-action policy},
  author={Hou, Zhi and Zhang, Tianyi and Xiong, Yuwen and Duan, Haonan and Pu, Hengjun and Tong, Ronglei and Zhao, Chengyang and Zhu, Xizhou and Qiao, Yu and Dai, Jifeng and Chen, Yuntao},
  journal={arXiv preprint arXiv:2503.19757},
  year={2025}
}

@article{bu2025univla,
  title={Univla: Learning to act anywhere with task-centric latent actions},
  author={Bu, Qingwen and Yang, Yanting and Cai, Jisong and Gao, Shenyuan and Ren, Guanghui and Yao, Maoqing and Luo, Ping and Li, Hongyang},
  journal={arXiv preprint arXiv:2505.06111},
  year={2025}
}

@article{guo2025glad,
  title={GLaD: Geometric Latent Distillation for Vision-Language-Action Models},
  author={Guo, Minghao and Cao, Meng and Tao, Jiachen and Xu, Rongtao and Yan, Yan and Liang, Xiaodan and Laptev, Ivan and Chang, Xiaojun},
  journal={arXiv preprint arXiv:2512.09619},
  year={2025}
}

@inproceedings{gr00tn1_2025,
  archivePrefix = {arxiv},
  eprint     = {2503.14734},
  title      = {{GR00T} {N1}: An Open Foundation Model for Generalist Humanoid Robots},
  author     = {Johan Bjorck and Fernando Castañeda, Nikita Cherniadev and Xingye Da and Runyu Ding and Linxi "Jim" Fan and Yu Fang and Dieter Fox and Fengyuan Hu and Spencer Huang and Joel Jang and Zhenyu Jiang and Jan Kautz and Kaushil Kundalia and Lawrence Lao and Zhiqi Li and Zongyu Lin and Kevin Lin and Guilin Liu and Edith Llontop and Loic Magne and Ajay Mandlekar and Avnish Narayan and Soroush Nasiriany and Scott Reed and You Liang Tan and Guanzhi Wang and Zu Wang and Jing Wang and Qi Wang and Jiannan Xiang and Yuqi Xie and Yinzhen Xu and Zhenjia Xu and Seonghyeon Ye and Zhiding Yu and Ao Zhang and Hao Zhang and Yizhou Zhao and Ruijie Zheng and Yuke Zhu},
  month      = {March},
  year       = {2025},
  booktitle  = {ArXiv Preprint},
}

@article{chen2025internvla,
  title={Internvla-m1: A spatially guided vision-language-action framework for generalist robot policy},
  author={Chen, Xinyi and Chen, Yilun and Fu, Yanwei and Gao, Ning and Jia, Jiaya and Jin, Weiyang and Li, Hao and Mu, Yao and Pang, Jiangmiao and Qiao, Yu and others},
  journal={arXiv preprint arXiv:2510.13778},
  year={2025}
}

@article{cen2025worldvla,
  title={WorldVLA: Towards Autoregressive Action World Model},
  author={Cen, Jun and Yu, Chaohui and Yuan, Hangjie and Jiang, Yuming and Huang, Siteng and Guo, Jiayan and Li, Xin and Song, Yibing and Luo, Hao and Wang, Fan and others},
  journal={arXiv preprint arXiv:2506.21539},
  year={2025}
}

@article{lin2025vote,
  title={Vote: vision-language-action optimization with trajectory ensemble voting},
  author={Lin, Juyi and Taherin, Amir and Akbari, Arash and Akbari, Arman and Lu, Lei and Chen, Guangyu and Padir, Taskin and Yang, Xiaomeng and Chen, Weiwei and Li, Yiqian and others},
  journal={arXiv preprint arXiv:2507.05116},
  year={2025}
}

@inproceedings{lee2025vlsi,
  title={VLsI: Verbalized Layers-to-Interactions from Large to Small Vision Language Models},
  author={Lee, Byung-Kwan and Hachiuma, Ryo and Wang, Yu-Chiang Frank and Ro, Yong Man and Wu, Yueh-Hua},
  booktitle={Proceedings of the Computer Vision and Pattern Recognition Conference},
  pages={29545--29557},
  year={2025}
}

@article{huang2025thinkact,
  title={Thinkact: Vision-language-action reasoning via reinforced visual latent planning},
  author={Huang, Chi-Pin and Wu, Yueh-Hua and Chen, Min-Hung and Wang, Yu-Chiang Frank and Yang, Fu-En},
  journal={arXiv preprint arXiv:2507.16815},
  year={2025}
}

@article{shi2025memoryvla,
  title={Memoryvla: Perceptual-cognitive memory in vision-language-action models for robotic manipulation},
  author={Shi, Hao and Xie, Bin and Liu, Yingfei and Sun, Lin and Liu, Fengrong and Wang, Tiancai and Zhou, Erjin and Fan, Haoqiang and Zhang, Xiangyu and Huang, Gao},
  journal={arXiv preprint arXiv:2508.19236},
  year={2025}
}

@article{chen2022improved,
  title={Improved feature distillation via projector ensemble},
  author={Chen, Yudong and Wang, Sen and Liu, Jiajun and Xu, Xuwei and de Hoog, Frank and Huang, Zi},
  journal={Advances in Neural Information Processing Systems},
  volume={35},
  pages={12084--12095},
  year={2022}
}

@article{kusupati2022matryoshka,
  title={Matryoshka representation learning},
  author={Kusupati, Aditya and Bhatt, Gantavya and Rege, Aniket and Wallingford, Matthew and Sinha, Aditya and Ramanujan, Vivek and Howard-Snyder, William and Chen, Kaifeng and Kakade, Sham and Jain, Prateek and others},
  journal={Advances in Neural Information Processing Systems},
  volume={35},
  pages={30233--30249},
  year={2022}
}

@article{skean2025layer,
  title={Layer by layer: Uncovering hidden representations in language models},
  author={Skean, Oscar and Arefin, Md Rifat and Zhao, Dan and Patel, Niket and Naghiyev, Jalal and LeCun, Yann and Shwartz-Ziv, Ravid},
  journal={arXiv preprint arXiv:2502.02013},
  year={2025}
}

@article{weinan2017proposal,
  title={A proposal on machine learning via dynamical systems},
  author={Weinan, Ee},
  journal={Communications in Mathematics and Statistics},
  volume={5},
  number={1},
  pages={1--11},
  year={2017},
  publisher={Springer Science and Business Media LLC}
}

@article{chang2017multi,
  title={Multi-level residual networks from dynamical systems view},
  author={Chang, Bo and Meng, Lili and Haber, Eldad and Tung, Frederick and Begert, David},
  journal={arXiv preprint arXiv:1710.10348},
  year={2017}
}

@article{hurst2024gpt,
  title={Gpt-4o system card},
  author={Hurst, Aaron and Lerer, Adam and Goucher, Adam P and Perelman, Adam and Ramesh, Aditya and Clark, Aidan and Ostrow, AJ and Welihinda, Akila and Hayes, Alan and Radford, Alec and others},
  journal={arXiv preprint arXiv:2410.21276},
  year={2024}
}

@article{aldaco2024aloha,
  title={Aloha 2: An enhanced low-cost hardware for bimanual teleoperation},
  author={Aldaco, Jorge and Armstrong, Travis and Baruch, Robert and Bingham, Jeff and Chan, Sanky and Draper, Kenneth and Dwibedi, Debidatta and Finn, Chelsea and Florence, Pete and Goodrich, Spencer and others},
  journal={arXiv preprint arXiv:2405.02292},
  year={2024}
}

@article{liu2024rdt,
  title={RDT-1B: a Diffusion Foundation Model for Bimanual Manipulation},
  author={Liu, Songming and Wu, Lingxuan and Li, Bangguo and Tan, Hengkai and Chen, Huayu and Wang, Zhengyi and Xu, Ke and Su, Hang and Zhu, Jun},
  journal={arXiv preprint arXiv:2410.07864},
  year={2024}
}

@article{fei2025libero,
  title={Libero-plus: In-depth robustness analysis of vision-language-action models},
  author={Fei, Senyu and Wang, Siyin and Shi, Junhao and Dai, Zihao and Cai, Jikun and Qian, Pengfang and Ji, Li and He, Xinzhe and Zhang, Shiduo and Fei, Zhaoye and others},
  journal={arXiv preprint arXiv:2510.13626},
  year={2025}
}

@article{cortes2012algorithms,
  title={Algorithms for learning kernels based on centered alignment},
  author={Cortes, Corinna and Mohri, Mehryar and Rostamizadeh, Afshin},
  journal={The Journal of Machine Learning Research},
  volume={13},
  number={1},
  pages={795--828},
  year={2012},
  publisher={JMLR. org}
}
